%% file: main_arxiv.tex
\documentclass[twoside,11pt]{article}
\usepackage{mhStyle}

\usepackage{times}
\usepackage{graphicx}
\usepackage{amsmath}
\usepackage{amssymb,url,xspace,caption,balance}
\usepackage{booktabs}
\usepackage[ruled,vlined]{algorithm2e}
\usepackage{algorithmic}

\usepackage{lipsum}

\usepackage{hyperref}  
\hypersetup{backref,colorlinks=true,linkcolor=red,citecolor=green,urlcolor=blue}

\graphicspath{{./}{./figures/}}
\DeclareGraphicsExtensions{.pdf,.jpeg,.png,.jpg}

\makeatletter
\DeclareRobustCommand\onedot{\futurelet\@let@token\@onedot}
\def\@onedot{\ifx\@let@token.\else.\null\fi\xspace}

\def\eg{\emph{e.g}\onedot} 
\def\ie{\emph{i.e}\onedot}

\def\etal{\emph{et al}\onedot}
\makeatother

\def\Vec#1{{\boldsymbol{#1}}}
\def\Mat#1{{\boldsymbol{#1}}}

\def\SPD#1{\mathcal{S}_{++}^{#1}}
\def\SYM#1{\mathcal{S}^{#1}}
\def\GRASS#1#2{\mathcal{G}({#1},{#2})}

\newcommand{\tr}{\mathop{\rm  Tr}\nolimits}
\newcommand{\DIAG}{\mbox{Diag\@\xspace}}

\begin{document}
\title{Dimensionality Reduction on SPD Manifolds: The Emergence of Geometry-Aware Methods}

\author
  {
  \name Mehrtash Harandi \email {mehrtash.harandi@anu.edu.au} \\
  \addr Australian National University, Canberra, Australia. \\ 
  \name Mathieu Salzmann \email {mathieu.salzmann@epfl.ch} \\
  \addr CVLab, \'{E}cole Polytechnique F\'{e}d\'{e}rale de Lausanne (EPFL), Switzerland. \\
  \name Richard~Hartley \email {richard.hartley@anu.edu.au} \\
  \addr Australian National University, Canberra, Australia. \\ 
  }

\maketitle
\begin{abstract}
Representing images and videos with Symmetric Positive Definite (SPD) matrices, and considering the Riemannian geometry of the resulting space, has been shown to yield high discriminative power in many visual recognition tasks. Unfortunately, computation on the Riemannian manifold of SPD matrices --especially of high-dimensional ones-- comes at a high cost that limits the applicability of existing techniques.  
In this paper, we introduce algorithms able to handle high-dimensional SPD matrices by constructing a lower-dimensional SPD manifold. To this end, we propose to model the mapping from the high-dimensional SPD manifold to the low-dimensional one
with an orthonormal projection. This lets us formulate dimensionality reduction as the problem of finding a projection that yields a low-dimensional manifold either with maximum discriminative power in the supervised scenario, or with maximum variance of the data in the unsupervised one. 
We show that learning can be expressed as an optimization problem on a Grassmann manifold and discuss fast solutions for special cases. 
Our evaluation on several classification tasks evidences that our approach leads to a significant accuracy gain over state-of-the-art methods.
\end{abstract}

\input{sec_intro}

\input{sec_background}
\input{sec_supervised_tl}
\input{sec_unsupervised_tl}

\input{sec_opt}

\input{sec_further_discussions}
\input{sec_related_work}

\input{sec_experiments}
\input{sec_conclusion}


\appendix
\input{sec_appendix}

\balance
\bibliography{references}

\end{document}

%% file: sec_intro.tex
\section{Introduction}
\label{sec:introduction}

Dimensionality Reduction (DR) is imperative in various disciplines of computer science, including 
machine learning and computer vision. Conventional methods, such as Principal Component Analysis (PCA) and Linear Discriminant Analysis (LDA), 
are specifically designed to work with real-valued vectors coming from a flat Euclidean space. In modern computer vision, however, the data and the mathematical models often naturally lie on Riemannian manifolds (\eg, subspaces form Grassmannian, 2D shapes lie on Kendall shape spaces~\cite{Kendall1984}). There has therefore been a growing need and interest to go beyond the extensively studied Euclidean spaces and analyze non-linear and curved Riemannian manifolds. In this context, a natural question arises: \emph{How can popular DR techniques be extended to curved Riemannian spaces?} A principled answer to this question will open the door to exploiting higher-dimensional, more discriminative features, and thus to improved accuracies in a wide range of applications involving classification and clustering.

This paper tackles the problem of dimensionality reduction on the space of Symmetric Positive Definite (SPD) matrices, \ie, the SPD manifold. In computer vision, SPD matrices have been successfully employed for a variety of tasks, such as analyzing medical imaging~\cite{Pennec_IJCV_2006}, segmenting images~\cite{Carreira_TPAMI_15} 
and recognizing textures~\cite{Tuzel_ECCV_2006,Harandi_TNNLS_2015}, pedestrians~\cite{Tuzel_PAMI_2008,Tosato_TPAMI_2013,Sadeep_CVPR_2013}, 
faces~\cite{Pang_TCSVT_2008,Wang_CVPR_2012_CDL,Sivalingam_TPAMI_2014}, and actions~\cite{Sanin_WACV_2013,Guo_TIP13}.

The set of SPD matrices is clearly not a vector space as it is not closed under addition and scalar product (\eg, multiplying a positive definite matrix with a negative scalar makes it negative definite). As such, analyzing SPD matrices through the geometry of Euclidean spaces, such as using the Frobenius inner product as a mean of measuring similarity, is not only unnatural, but also inadequate. This inadequacy has recently been demonstrated in computer vision by a large body of work, \eg,~\cite{Pennec_IJCV_2006,Tuzel_PAMI_2008,Sadeep_CVPR_2013}.
One striking example is the \textit{swelling effect} that occurs in diffusion tensor imaging (DTI), where a matrix represents the covariance of the local Brownian motion of water molecules~\cite{Pennec_IJCV_2006}-- when considering Euclidean geometry to interpolate between two diffusion tensors, the determinant of the intermediate matrices may become strictly larger than the determinants of both original matrices, which, from a physics point of view, is unacceptable.

A popular and geometric way to analyze SPD matrices is through the Riemannian structure induced by the Affine Invariant 
Riemannian Metric (AIRM)~\cite{Pennec_IJCV_2006}, which is usually referred to as SPD manifold. 
The geodesic distance induced by AIRM is related to the distance induced by the Fisher-Rao metric on the manifold of multivariate 
Gaussian distributions with fixed means (see for example~\cite{Atkinson1981}). It enjoys several properties, such as invariance to affine transformations, which are of particular interest in computer vision.

While the Riemannian structure induced by AIRM has been shown to overcome the limitations of Euclidean geometry to a great extent, the computational cost of the resulting techniques increases drastically with the dimension of the manifold (\ie, the size of the SPD matrices). As a consequence, with the exception of a few works that handle medium-sized features~\cite{Carreira_TPAMI_15,Wang_CVPR_2012_CDL}, previous studies have opted for low-dimensional SPD matrices (\eg, region covariance descriptors obtained from low-dimensional features). Clearly, and as evidenced by the recent feature-learning trends in computer vision, low-dimensional features are bound to be less powerful and discriminative. In other words, to match or even outperform state-of-the-art recognition systems on complex tasks, manifold-based methods will need to exploit high-dimensional SPD matrices. This paper introduces techniques to perform supervised and unsupervised DR methods dedicated to SPD manifolds, as illustrated by Fig.~\ref{fig:intro}.

\begin{figure}[!t]
\centering
	\includegraphics[width= 0.7\columnwidth,keepaspectratio]{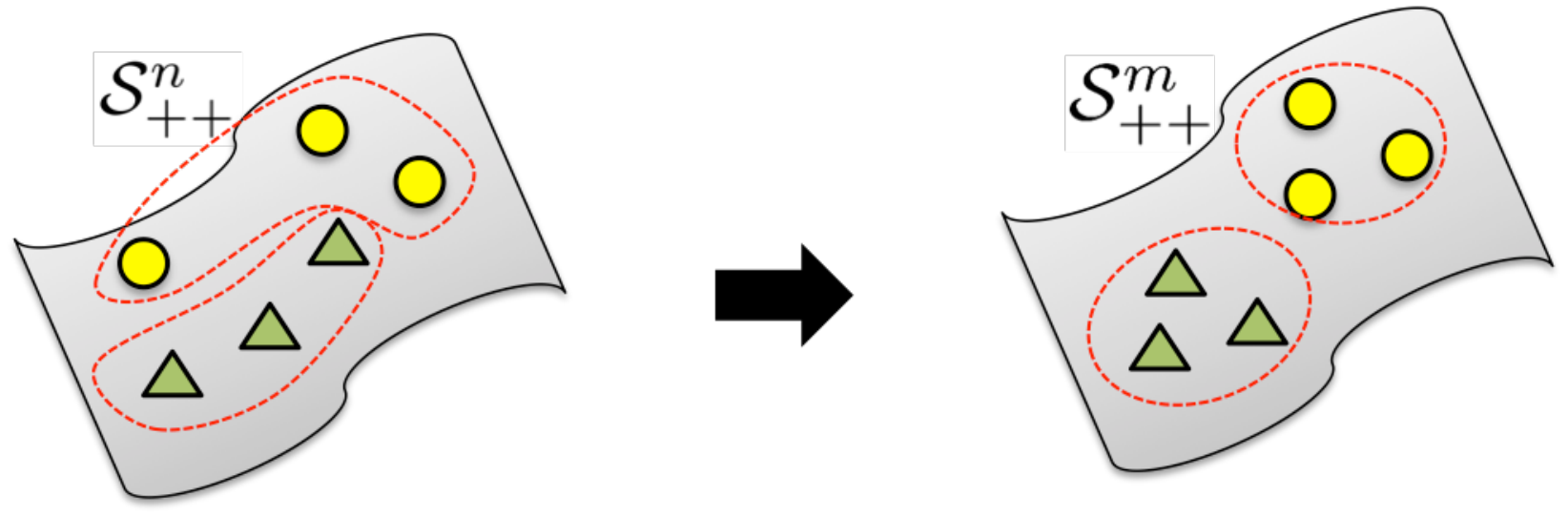}
	\caption{\small \textbf{Dimensionality Reduction on SPD Manifolds:}  
	Given data on a high-dimensional SPD manifold, where each sample represents an $n\times n$ SPD matrix, we learn a mapping 
	to a lower-dimensional SPD manifold. We consider both the supervised scenario, illustrated here, where the resulting $m\times m$ SPD matrices are clustered according to class labels, and the unsupervised one, where the resulting matrices have maximum variance.
    }
\label{fig:intro}
\end{figure}

More specifically, in the supervised scenario, we introduce an approach that constructs a lower-dimensional and more discriminative SPD manifold from a high-dimensional one. To this end, we encode the notion of discriminative power by pulling together the training samples from the same class while pushing apart those from different classes. We study three variants of this approach, where the distance is defined by either the AIRM, the Stein divergence~\cite{Sra_NIPS_2012}, or the Jeffrey divergence~\cite{Vemuri_CVPR_2004}. In particular, the latter two divergences are motivated by the fact that they share invariance properties with the AIRM while being faster to compute.

In the unsupervised scenario, we draw inspiration from the Maximum Variance 
Unfolding (MVU) algorithm~\cite{Weinberger_IJCV_2006}. That is, we introduce a method that maps a high-dimensional SPD manifold 
to a low-dimensional one, where the training matrices are furthest apart from their mean. As in the supervised case, we study three variants, that rely on the AIRM, the Stein divergence and the Jeffrey divergence, respectively~\cite{Cherian:PAMI:2013}.

We demonstrate the benefits of our approach on several classification and clustering tasks where 
the data can be represented with high-dimensional SPD matrices. In particular, our method outperforms state-of-the-art techniques on image-based material categorization and face recognition, and action recognition from 3D motion capture sequences. A Matlab implementation of our algorithms is available from the first author's webpage\footnote{%
This paper is an extended and revised version of
our earlier work~\cite{Harandi_ECCV_2014}. In addition to providing more insight on the
proposed methods, we extend our initial work by introducing an unsupervised DR algorithm and
deriving variants of our unsupervised/supervised DR methods based on the Jeffrey divergence. 
}.

%% file: sec_background.tex
\section{Background Theory}
\label{sec:background}

This section provides an brief review of the Riemannian geometry of SPD manifolds, as well as of Bregman divergences and their properties. 

\textbf{Notation:} 
Throughout the paper, bold capital letters denote matrices (\eg, $\Mat{X}$) and bold lower-case letters denote column vectors (\eg, $\Vec{x}$).
$\mathbf{I}_n$ is the $n \times n$ identity matrix.
$GL(n)$ denotes the  general linear group, \ie, the group of real invertible $n \times n$ matrices.
$\SYM{n}$ is the space of real $n \times n$ symmetric matrices.
$\SPD{n}$ and $\GRASS{p}{n}$ are the SPD and Grassmannian manifolds, respectively, and will be formally defined later.
$\mathrm{Diag} \left( \lambda_1,\lambda_2,\cdots,\lambda_n \right)$ is a diagonal matrix with the
real values $\lambda_1,\lambda_2,\cdots,\lambda_n$ as diagonal elements.
The principal matrix logarithm $\log(\cdot):\SPD{n} \to \SYM{n}$ is defined as 
\begin{equation}
	\log(\Mat{X}) = \sum\limits_{r=1}^{\infty}{\frac{(-1)^{r-1}}{r}\left(\Mat{X}-\mathbf{I}_n\right)^r} = 
	\Mat{U} \DIAG\left(\log(\lambda_i) \right) \Mat{U}^T,
	\label{eqn:principal_log_oper}
\end{equation}
with \mbox{$\Mat{X} = \Mat{U} \DIAG\left(\lambda_i \right) \Mat{U}^T $}.
Similarly, the matrix exponential $\exp(\cdot):\SYM{n} \to \SPD{n}$ is defined as
\begin{equation}
	\exp(\Mat{X}) = \sum\limits_{r=0}^{\infty}{\frac{1}{r!}\Mat{X}^r} = 
	\Mat{U} \DIAG\left(\exp(\lambda_i) \right) \Mat{U}^T,
	\label{eqn:principal_exp_oper}
\end{equation}
with \mbox{$\Mat{X} = \Mat{U} \DIAG\left(\lambda_i \right) \Mat{U}^T $}. 

\subsection{The Geometry of SPD Manifolds}
An $n \times n$ real SPD matrix $\Mat{X}$ has the property that $\Vec{v}^T\Mat{X}\Vec{v} > 0$ for all non-zero $\Vec{v} \in \mathbb{R}^n$. 
The space of $n \times n$ SPD matrices, denoted by $\SPD{n}$ forms the interior of a convex cone in the $n(n+1)/2$-dimensional Euclidean space.
$\SPD{n}$ is mostly studied when endowed with the Affine Invariant Riemannian Metric (AIRM)~\cite{Pennec_IJCV_2006}, defined as

\begin{align}
	\langle \Vec{v}, \Vec{w} \rangle_\Mat{P} &\triangleq  \langle \Mat{P}^{-1/2}\Vec{v}\Mat{P}^{-1/2}, \Mat{P}^{-1/2}\Vec{w}\Mat{P}^{-1/2} \rangle
	\notag \\
	&= \tr \left( \Mat{P}^{-1} \Vec{v} \Mat{P}^{-1} \Vec{w}\right)\;,
	\label{eqn:AIRM_equ}
\end{align}

\noindent
for $\Mat{P} \in \SPD{n}$ and $\Vec{v},\Vec{w} \in T_{\Mat{P}}{\SPD{n}}$, where $T_{\Mat{P}}{\mathcal{M}}$ denotes the 
tangent space of the manifold $\mathcal{M}$ at $\Mat{P}$. 
This metric induces the following geodesic distance between points
$\Mat{X}$ and $\Mat{Y}$:

\begin{equation}
\delta_R(\Mat{X},\Mat{Y}) = \|\log(\Mat{X}^{-1/2}\Mat{Y}\Mat{X}^{-1/2})\|_F\;.
\label{eqn:geodesic_distance}
\end{equation}
The AIRM has several useful properties such as invariance to affine transformations (as the name implies), \ie, 
$\delta_R(\Mat{X},\Mat{Y}) = \delta_R(\Mat{A}\Mat{X}\Mat{A}^T,\Mat{A}\Mat{Y}\Mat{A}^T)$. 
For in-depth discussions of the AIRM, we refer the interested reader to~\cite{Pennec_IJCV_2006} 
and~\cite{BHATIA_2007}.

\subsection{Bregman Divergences}
\label{sec:subsec_bregman}

We now introduce two divergences derived from the Bregman matrix divergence, namely the Jeffrey and Stein divergences.
Below, we discuss their properties and establish some connections with the AIRM, which motivated our choice of these divergences in our DR formulations.

\begin{definition} \label{def:bregman_divergence}
Let $\zeta : \mathcal{S}_{++}^{n} \rightarrow \mathbb{R}$ be a strictly
convex and differentiable function defined on the symmetric positive cone $\SPD{n}$. The Bregman 
matrix divergence 	$d_\zeta : \SPD{n} \times \SPD{n} \rightarrow [0,\infty)$ is defined as
\begin{equation}
	\label{eqn:Bregman_Div}
	d_\zeta(\Mat{X},\Mat{Y}) = \zeta(\Mat{X}) - \zeta(\Mat{Y}) - \langle \nabla_{\zeta} (\Mat{Y}) , \Mat{X} - \Mat{Y} \rangle \; ,
\end{equation}
where
$\langle \Mat{X} , \Mat{Y} \rangle \mbox{=} \tr \big( \Mat{X}^T \Mat{Y} \big) $ is the Frobenius inner product,
and	$ \nabla_{\zeta}(\Mat{Y})$ represents the gradient of $\zeta$ evaluated at $\Mat{Y}$.
\end{definition}
	
The Bregman divergence is asymmetric, non-negative, and definite (\ie , $d_\zeta(\Mat{X},\Mat{Y}) = 0, \; \text{iff}~ \Mat{X} =\Mat{Y}$).
While the Bregman divergence enjoys a variety of useful properties~\cite{Kulis:2009:JMLR},
its asymmetric behavior is often a hindrance.	
In this paper we are interested in two types of symmetrized Bregman divergences, namely the {\it Stein}  and the {\it Jeffrey} divergences.

\begin{definition} \label{def:stein_divergence}
The Stein, or $S$, divergence (also known as Jensen-Bregman LogDet divergence~\cite{Cherian:PAMI:2013})
is obtained from the Bregman divergence of Eq.~\ref{eqn:Bregman_Div}
by using $\zeta(\Mat{X}) = - \log \det(\Mat{X})$ as seed function and by {\it Jensen-Shannon} symmetrization. This yields
\begin{align}
	\delta_S^2(\Mat{X},\Mat{Y}) &\triangleq \frac{1}{2}	
   	d_{\zeta}\left( \Mat{X},\frac{\Mat{X}+\Mat{Y}}{2} \right) +
	\frac{1}{2} d_{\zeta}\left( \Mat{Y},\frac{\Mat{X}+\Mat{Y}}{2} \right) \notag\\
    &= \log \det\bigg( \frac{\Mat{X}+\Mat{Y}}{2} \bigg)
    - \frac{1}{2}  \log \det ( \Mat{X}\Mat{Y})  \;. 
    \label{eqn:Stein_Div}
\end{align}%
\end{definition}
	 
\begin{definition} \label{def:kl_divergence}
The Jeffrey, or $J$, divergence (also known as symmetric KL divergence) is obtained from the Bregman divergence of 	
Eq.~\ref{eqn:Bregman_Div} by using $\zeta(\Mat{X}) = -\ln \det(\Mat{X})$ as seed function and by direct symmetrization.  This yields
\begin{align}
	\delta_J^2(\Mat{X},\Mat{Y}) &\triangleq \frac{1}{2} d_\zeta(\Mat{X},\Mat{Y}) + \frac{1}{2}d_\zeta(\Mat{Y},\Mat{X}) \notag\\
	&= \frac{1}{2}\tr (\Mat{X}^{-1}\Mat{Y}) - \frac{1}{2}\log \det(\Mat{X}^{-1}\Mat{Y}) \notag\\
	&+   \frac{1}{2}\tr (\Mat{Y}^{-1}\Mat{X}) - \frac{1}{2}\log \det(\Mat{Y}^{-1}\Mat{X}) -n \notag\\			  
	&= \frac{1}{2}\tr (\Mat{X}^{-1}\Mat{Y}) + \frac{1}{2}\tr (\Mat{Y}^{-1}\Mat{X}) -n \;.
	\label{eqn:KL_Div}
\end{align}
\end{definition}

The $S$ and $J$ divergences have a variety of properties akin to those of the AIRM.
Below, we present the properties that motivated us to perform DR on $\SPD{n}$ using such divergences.

\subsubsection*{\textbf{Invariance properties}}
An especially attractive property for the computer vision community is the invariance of the $J$ and $S$ divergences to affine transforms.
More specifically (and similarly to the AIRM), for $\Mat{A} \in \rm{GL}(n)$, we have
\begin{align*}
	\delta_S^2(\Mat{X},\Mat{Y}) & = \delta_S^2(\Mat{A}\Mat{X}\Mat{A}^T,\Mat{A}\Mat{Y}\Mat{A}^T). \\
	\delta_J^2(\Mat{X},\Mat{Y}) & = \delta_J^2(\Mat{A}\Mat{X}\Mat{A}^T,\Mat{A}\Mat{Y}\Mat{A}^T).
\end{align*}%
\noindent
Furthermore, both divergences are invariant to inversion, \ie, 

\begin{align*}
	J(\Mat{X},\Mat{Y}) & = J(\Mat{X}^{-1},\Mat{Y}^{-1}) \\
	S(\Mat{X},\Mat{Y}) & = S(\Mat{X}^{-1},\Mat{Y}^{-1}).
\end{align*}%
\noindent
Proofs for the above statements can be readily obtained by plugging the affine representations (\eg $\Mat{A}\Mat{X}\Mat{A}^T$)
or inverses into the definition of the $J$ and $S$ divergences.

\subsubsection*{\textbf{Equality of curve lengths}}
Given a curve $\gamma:[0,1] \to \SPD{n}$, let $L_R$, $L_S$ and $L_J$ denote the length of $\gamma$ under 
AIRM, Stein and J-divergence, respectively. Then $L_R = 2\sqrt{2} L_S$ and $L_R = \sqrt{2} L_J$.
The proof for the case of $L_S$ is given in~\cite{Harandi_ECCV_2014} and for the $L_J$ is relegated to the supplementary material.

Beyond the previous two properties, the $S$ divergence also enjoys the following Hilbert space embedding property, 
which does not hold for AIRM~\cite{Sadeep_CVPR_2013}.

\subsubsection*{\textbf{Hilbert space embedding}}
The $S$ divergence admits a Hilbert space embedding in the form of a Radial Basis Function (RBF) kernel~\cite{Sra_NIPS_2012}. 
More specifically, the kernel 
\begin{equation}
	k_S(\Mat{X},\Mat{Y}) = \exp \{ -\beta \delta_S^2(\Mat{X},\Mat{Y}) \},
	\label{eqn:kernel_s_div}
\end{equation}
is positive definite for 
\begin{equation}
    \beta \in \left \{ \frac{1}{2},\frac{2}{2}, \cdots, \frac{n-1}{2} \right \}
    \cup \left (\frac{1}{2}(n-1),\infty \right )\;.
    \label{eqn:Stein_Krnl2}
\end{equation}

The kernel $k_J(\cdot,\cdot) = \exp \{-\beta \delta_J^2(\Mat{X},\Mat{Y})\}$ has been previously considered to be positive definite (\eg, equations 5 and 6 in~\cite{Moreno_2003}). However, we find that this is not the case as can be seen by the following counter example
\begin{scriptsize}
\begin{align*}
\Mat{X}_1 = \begin{bmatrix}
72		&1\\
1		&88
\end{bmatrix},~
\Mat{X}_2 = \begin{bmatrix}
123		&-10\\
-10	&66
\end{bmatrix},~
\Mat{X}_3 = \begin{bmatrix}
51		&5\\
5		&109
\end{bmatrix}.
\end{align*}
\end{scriptsize}
Here, the matrix $[\Mat{K}]_{i,j} = k_J(\Mat{X}_i,\Mat{X}_j)$ has a negative eigenvalue for $\beta = 1/4$. 
With the mathematical tools discussed in this section, we can now turn to developing our DR algorithms for SPD matrices. In the following sections, we first start by introducing our approach to tackling supervised DR on SPD manifolds and then discuss the unsupervised scenario.

%% file: sec_supervised_tl.tex
\section{DR on SPD Manifolds}
\label{sec:DR}
In this section, we describe our approach to learning an embedding of high-dimensional SPD matrices to a more discriminative, low-dimensional SPD manifold. In doing so, we propose to learn the parameters $\Mat{W} \in \mathbb{R}^{n \times m}$, $m<n$, of a generic mapping 
$f_\Mat{W}: \SPD{n}  \rightarrow \SPD{m}$, which we define as 
\begin{equation}
	f_\Mat{W}(\Mat{X}) = \Mat{W}^T\Mat{X}\Mat{W}.
	\label{eqn:generic_mapping}
\end{equation}
Clearly, for a full rank matrix $\Mat{W}$, if $\SPD{n} \ni \Mat{X} \succ 0$ then $\SPD{m} \ni \Mat{W}^T\Mat{X}\Mat{W} \succ 0$. 
Given a set of training SPD matrices $\mathcal{X} = \left\{\Mat{X}_1,\cdots, \Mat{X}_p \right\}$, where each matrix $\Mat{X}_i \in \SPD{n}$, our goal is to find the transformation $\Mat{W}$ such that the resulting low-dimensional SPD manifold preserves some interesting structure of the original data. In the remainder of this section, we discuss two different such structures: one coming from the availability of class labels, and one derived from unsupervised data.

\subsection{Supervised Dimensionality Reduction}
\label{sec:supervised_dr}
Let us first assume that each point $\Mat{X}_i \in \SPD{n}$ belongs to one of $C$ possible classes 
and denote its class label by $y_i$.  In this scenario, we propose to encode the structure of the data via an affinity function $a: \SPD{n} \times \SPD{n} \to \mathbb{R}$. That is $a(\Mat{X},\Mat{Y})$ measures some notion of affinity between matrices $\Mat{X}$ and $\Mat{Y}$, and may be negative. In particular, we make use of the class labels to build $a(\cdot,\cdot)$\footnote{Note that the framework developed in this section could also apply to the unsupervised and semi-supervised settings by changing the definition of the affinity function accordingly.} and define an affinity function that encodes the notions of intra-class and inter-class distances. In short, our goal is to find a mapping
that minimizes the intra-class distances while simultaneously maximizing the inter-class distances (\ie, a discriminative mapping). 

More specifically, we make use of notions of within-class similarity $g_w: \SPD{n} \times \SPD{n} \to \mathbb{R}_+$
and between-class similarity $g_b: \SPD{n} \times \SPD{n} \to \mathbb{R}_+$ to compute the affinity between two SPD matrices. In particular, we define $g_w(\cdot,\cdot)$ and $g_b(\cdot,\cdot)$ to be binary functions given by
\begin{equation}
   	g_w(\Mat{X}_i,\Mat{X}_j) =
   	\left\{
 		\begin{matrix}
       	1, & \mbox{if} \; \Mat{X}_i \in N_w(\Mat{X}_j ) \; \mbox{~or~} \; \Mat{X}_j \in N_w(\Mat{X}_i ) \\
   		0, & \mbox{otherwise}
   		\end{matrix}
   	\right.
\label{eqn:Gw}
\end{equation}%
\noindent
\begin{equation}
   	g_b(\Mat{X}_i,\Mat{X}_j) =
   	\left\{
   	\begin{matrix}
       	1, & \mbox{if} \; \Mat{X}_i \in N_b(\Mat{X}_j ) \; \mbox{~or~} \; \Mat{X}_j \in N_b(\Mat{X}_i ) \\
       	0, & \mbox{otherwise}
   	\end{matrix}
   	\right.
   	\label{eqn:Gb}
\end{equation}%
\noindent
where $N_w(\Mat{X}_i)$ is the set of $\nu_w$ nearest neighbours of $\Mat{X}_i$ that share the same label as $y_i$, 
and $N_b(\Mat{X}_i)$ contains the $\nu_b$ nearest neighbours of $\Mat{X}_i$
having different labels. Note that nearest neighbours are computed according to the AIRM, the Stein divergence, or the Jeffrey divergence. The affinity function $a(\cdot,\cdot)$ is then defined as
\begin{equation}
	a(\Mat{X}_i,\Mat{X}_j) = g_w(\Mat{X}_i,\Mat{X}_j) - g_b(\Mat{X}_i,\Mat{X}_j) \;,
	\label{eqn:affinity_Gw_Gb_combined}
\end{equation}
which resembles the Maximum Margin Criterion (MMC) of~\cite{LI_TNN_2006}. 

Having  $a(\Mat{X},\Mat{Y})$  at our disposal, we propose to search for an embedding such that the affinity between pairs of SPD matrices is reflected by a measure of similarity on the low-dimensional SPD manifold. In particular, we make use of the AIRM, the Stein divergence, or the Jeffrey divergence to encode similarity between SPD matrices. This lets us write a cost function of the form
\begin{equation}
	\label{eqn:dtl_cost_fun}
	L(\Mat{W}) = \sum_{\substack{i,j=1 \\ j \neq i}}^p a(\Mat{X}_i,\Mat{X}_j) 
	\delta^2 \left(\Mat{W}^T\Mat{X}_i\Mat{W},\Mat{W}^T\Mat{X}_j\Mat{W}\right)\;,
\end{equation}
where $\delta$ is $\delta_R$, $\delta_S$ or $\delta_J$.		
To avoid degeneracies and ensure that the resulting embedding forms a valid SPD manifold, \ie, $\Mat{W}^T\Mat{X}\Mat{W} \succ 0,~ \forall \Mat{X} \in \SPD{n}$, we need $\Mat{W}$ to have full rank. Here, we enforce this requirement by imposing the unitary constraints $\Mat{W}^T\Mat{W} = \mathbf{I}_m$. Note that, with the affine invariance property, this entails no loss of generality. 
Indeed, any full rank matrix $\tilde{\Mat{W}}$ can be expressed as $\Mat{W}\Mat{M}$, with $\Mat{W}$ an orthonormal matrix and $\Mat{M} \in \mathrm{GL}(m)$. The affine invariance property of the metric therefore guarantees that
\begin{equation*}
	L(\tilde{\Mat{W}}) = L(\Mat{W}\Mat{M})  = L(\Mat{W})\;.
\end{equation*}	

As a result, learning can be expressed as the minimization problem
\begin{align}
	\label{eqn:tl_main_equ}
	\Mat{W}^{\ast} = ~&\underset{\Mat{W} \in  \mathbb{R}^{n \times m}}{\arg\min}
	\sum_{\substack{i,j=1 \\ j \neq i}}^p a(\Mat{X}_i,\Mat{X}_j) \delta^2\left( \Mat{W}^T\Mat{X}_i\Mat{W}, \Mat{W}^T\Mat{X}_j\Mat{W} \right) \notag\\
	&{\rm s.t.}~\Mat{W}^T\Mat{W} = \mathbf{I}_m\;.
\end{align}

As will be discussed in Section~\ref{sec:opt},~\eqref{eqn:tl_main_equ} is an optimization problem on a Grassmann manifold, and can thus be solved by Newton-type methods on the Grassmannian $\GRASS{m}{n}$. To this end, we need to compute the Jacobian of $\delta^2(\cdot,\cdot)$ with respect to $\Mat{W}$.
For the $S$ divergence, this Jacobian matrix, denoted by $D_\Mat{W}(\cdot)$ hereafter, can be obtained by noting that (see Eq. 53 in~\cite{Petersen_2012})
\begin{equation}
	D_\Mat{W}  \log\det\big(\Mat{W}^T\Mat{X}\Mat{W}\big)  = 2\Mat{X}\Mat{W}\big( \Mat{W}^T\Mat{X}\Mat{W}\big)^{-1}\;.
	\label{eqn:logdet_gradient}
\end{equation}	

\noindent
This lets us express the Jacobian of the Stein divergence as
\begin{align}
	\label{eqn:jacobian_tl_stein}
	&D_\Mat{W} \delta_S^2 \big( \Mat{W}^T\Mat{X}\Mat{W}, \Mat{W}^T\Mat{Y}\Mat{W} \big) = -\Mat{X}\Mat{W}(\Mat{W}^T\Mat{X}\Mat{W})^{-1}\\
	&-\Mat{Y}\Mat{W}(\Mat{W}^T\Mat{Y}\Mat{W})^{-1} + (\Mat{X}+\Mat{Y})\Mat{W}
	\big(\Mat{W}^T\frac{\Mat{X}+\Mat{Y}}{2}\Mat{W} \big)^{-1} \notag\;.
\end{align}

For the $J$ divergence, the Jacobian can be obtained by noting that (see Eq. 126 in~\cite{Petersen_2012})
\begin{align}
	&D_\Mat{W}  \tr \Big(\Mat{W}^T\Mat{X}\Mat{W} \big(\Mat{W}^T\Mat{Y}\Mat{W}\big)^{-1}\Big)  = 
	2\Mat{X}\Mat{W}\big( \Mat{W}^T\Mat{Y}\Mat{W}\big)^{-1} \notag\\ 
	&-2\Mat{Y}\Mat{W}\big( \Mat{W}^T\Mat{Y}\Mat{W}\big)^{-1}\big( \Mat{W}^T\Mat{X}\Mat{W}\big)\big( \Mat{W}^T\Mat{Y}\Mat{W}\big)^{-1}  \;,
	\label{eqn:trace_inv_gradient}
\end{align}	
which leads to 
\begin{align}
	\label{eqn:jacobian_tl_jeffrey}
	&D_\Mat{W} \delta_J^2 \big( \Mat{W}^T\Mat{X}\Mat{W}, \Mat{W}^T\Mat{Y}\Mat{W} \big) = 
	\\ 
	&	\Mat{X}\Mat{W} \hspace{-0.3ex} \Big( \hspace{-0.2ex} \big(\Mat{W}^T \hspace{-0.5ex} 
	\Mat{Y}\Mat{W}\big)^{\hspace{-0.5ex}-1} \hspace{-1.7ex} -		
	\big( \Mat{W}^T \hspace{-0.5ex} \Mat{X}\Mat{W}\big)^{\hspace{-0.5ex}-1} 
	\Mat{W}^T \hspace{-0.5ex} \Mat{Y}\Mat{W} 
	\big( \Mat{W}^T \hspace{-0.5ex} \Mat{X}\Mat{W}\big)^{\hspace{-0.5ex}-1} \Big) + \notag \\
	&\Mat{Y}\Mat{W} \hspace{-0.3ex} \Big( \hspace{-0.2ex} \big(\Mat{W}^T \hspace{-0.5ex} 
	\Mat{X}\Mat{W}\big)^{\hspace{-0.5ex}-1} \hspace{-1.7ex} -		
	\big( \Mat{W}^T \hspace{-0.5ex} \Mat{Y}\Mat{W}\big)^{\hspace{-0.5ex}-1} \Mat{W}^T \hspace{-0.5ex} \Mat{X}\Mat{W} 
	\big( \Mat{W}^T \hspace{-0.5ex} \Mat{Y}\Mat{W}\big)^{\hspace{-0.5ex}-1} \Big) \;. \notag
\end{align}

For the AIRM, we can exploit the fact that $\tr\left( \log(\Mat{X}) \right) = \ln \det\big ( \Mat{X} \big), \forall\Mat{X} \in \SPD{n}$. We can then derive the Jacobian by utilizing Eq.~\ref{eqn:logdet_gradient}, which yields
\begin{align}
	\label{eqn:jacobian_tl_airm}
	&D_\Mat{W}{\delta^2_{R}} \left( \Mat{W}^T\Mat{X}\Mat{W}, \Mat{W}^T\Mat{Y}\Mat{W} \right)  = \\
	&4\Big(\Mat{X}\Mat{W}(\Mat{W}^T\Mat{X}\Mat{W})^{-1} -
			\Mat{Y}\Mat{W}(\Mat{W}^T\Mat{Y}\Mat{W})^{-1}\Big) \notag \\
	&\times	\log\Big(\Mat{W}^T\Mat{X}\Mat{W}\big(\Mat{W}^T\Mat{Y}\Mat{W}\big)^{-1}\Big)\,.	\notag
\end{align}

Our supervised DR method for SPD matrices is summarized in Algorithm~\ref{alg:tl_alg}, where $\tau(\Mat{H},\Mat{W}_0,\Mat{W}_1)$ denotes the parallel transport of tangent vector $\Mat{H}$ from $\Mat{W}_0$ to $\Mat{W}_1$ (see Section~\ref{sec:opt} for details). 

\begin{algorithm}[!tb]		
	\caption{\small Supervised SPD DR}
	\label{alg:tl_alg}
	\begin{algorithmic}
	\vspace{0.1cm}
	\STATE {\bfseries Input:}\\
	{A set of SPD matrices $\{ \Mat{X}_i \}_{i=1}^p,~\Mat{X}_i~\in~\SPD{n}$.\\ 
	The corresponding labels $\{ y_i \}_{i=1}^p,~y_i \in \{1,2,\cdots,C\}$.\\
    The dimensionality $m$ of the induced manifold.
	}
	\vspace{0.1cm}
	\STATE {\bfseries Output:}\\
	{The mapping $\Mat{W} \in \mathcal{G}(m,n)$
	}
	\vspace{0.25cm}
	\STATE Generate $a(\Mat{X}_i,\Mat{X}_j)$ using \eqref{eqn:affinity_Gw_Gb_combined}
	\STATE $\Mat{W}_{old} \gets \mathbf{I}_{n\times m}$ (\ie, the truncated identity matrix)
	\STATE $\Mat{W} \gets \Mat{W}_{old}$
    \STATE $\Mat{H}_{old} \gets \Mat{0}$
	\REPEAT		
	    \STATE $\Mat{H} \gets -\nabla_{\Mat{W}} L(\Mat{W}) +\eta \tau(\Mat{H}_{old},\Mat{W}_{old},\Mat{W})$
		\STATE Line search along the geodesic starting from $\Mat{W}$ in the direction $\Mat{H}$
		to find \mbox{$\Mat{W}^\ast = \underset{\Mat{W}}{\rm{argmin}}~L(\Mat{W})$}			
		\STATE $\Mat{H}_{old} \gets \Mat{H}$
		\STATE $\Mat{W}_{old} \gets \Mat{W}$	
		\STATE $\Mat{W} \gets \Mat{W}^\ast$
	\UNTIL{convergence}
	\end{algorithmic}
\end{algorithm} 

%% file: sec_unsupervised_tl.tex
\subsection{Unsupervised Dimensionality Reduction}
\label{sec:unsupervised_tl}

We now turn to the scenario where we do {\it not} have access to the labels of the training samples. In other words, our training data only consists of  a set of SPD matrices  $\{\Mat{X}_i\}_{i = 1}^{p},\;\Mat{X}_i \in \SPD{n}$. To tackle this unsupervised DR scenario, we draw inspiration from algorithms, such as PCA and MVU~\cite{Weinberger_IJCV_2006}. These algorithms search for a low-dimensional latent space where the points have maximum variance, \ie, collectively have maximum distance to their mean. 

Here, we follow the same intuition, but with the goal of mapping high-dimensional SPD matrices to lower-dimensional ones. To this end, we express unsupervised DR on SPD manifolds as the optimization problem
\begin{align}
	\label{eqn:pca_tl}
	\Mat{W}^{\ast} = ~&\underset{\Mat{W} \in  \mathbb{R}^{n \times m}}{\arg\max}
	\sum_{i = 1}^p \delta^2\left( \Mat{W}^T\Mat{X}_i\Mat{W}, \Mat{W}^T\Mat{M}\Mat{W} \right) \notag\\
	&{\rm s.t.}~\Mat{W}^T\Mat{W} = \mathbf{I}_m\;,
\end{align}
with $\Mat{M}$ the mean of  $\{\Mat{X}_i\}_{i = 1}^{p}$ with respect to the metric $\delta$. 
As in the supervised case, and as discussed in more details in Section~\ref{sec:opt},~\eqref{eqn:pca_tl} corresponds to an optimization problem in the Grassmann manifold. We therefore again opt for a Newton-type method on the Grassmannian to (approximately) solve it. Note that the gradient of the objective function has essentially the same form as in the supervised case, and can thus be easily obtained from Eq.~\eqref{eqn:jacobian_tl_stein},
Eq.~\eqref{eqn:jacobian_tl_jeffrey} and Eq.~\eqref{eqn:jacobian_tl_airm} for the Stein divergence, the AIRM and J-divergence, respectively.

As mentioned above,~\eqref{eqn:pca_tl} depends on the mean of the training samples. Since these samples lie on a manifold, special care must be taken to compute their means. In particular, we make use of the Fr\'{e}chet formulation to obtain $\Mat{M}$. This can be expressed as
\begin{align}
	\Mat{M}^\ast \triangleq \underset{\Mat{M} \in \SPD{n}}{\arg\min}~~\sum_{i=1}^p \delta^2(\Mat{X}_i,\Mat{M})\;.
	\label{eqn:frechet_mean}
\end{align}
For the AIRM, this is equivalent to computing the Riemannian (Karcher) mean by exploiting the exponential 
and logarithm maps~\cite{Pennec_IJCV_2006}. For the Stein metric, we make use of the iterative Convex Concave Procedure (CCCP) introduced in~\cite{Cherian:PAMI:2013}. For the Jeffrey divergence, we show below that, unlike the AIRM and the Stein divergence, the Fr\'{e}chet mean can be computed analytically.

\begin{theorem}
The Fr\'{e}chet mean of a set of points $\big\{ \Mat{X}_i \big\}_{i=1}^{p},\; \Mat{X}_i \in \SPD{n}$, based on the Jeffrey metric, \ie, the minimizer of Eq.~\ref{eqn:frechet_mean} for $\delta^2(\cdot,\cdot) = \delta_{J}^2$, is given by
\begin{align}
	&\Mat{M}^\ast =  {\bf L}^{-1/2}\big({\bf L}^{1/2} \Mat{\Gamma} {\bf L}^{1/2}\big)^{1/2}{\bf L}^{-1/2}	
	\label{eqn:closed_mean_J}
\end{align}
with ${\bf L} = \sum_{i = 1}^{p}\Mat{X}_i^{-1}$ and $\Mat{\Gamma} = \sum_{i = 1}^{N}\Mat{X}_i$.

\begin{proof}
To prove this theorem, let us first we recall that, for $\Mat{A} \succ 0$ and $\Mat{B} \succeq 0$, a quadratic equation of the form $\Mat{X}\Mat{A}\Mat{X} = \Mat{B}$, called 
a \textit{Riccati} equation, has only one positive definite solution of the form~\cite{BHATIA_2007}
\begin{align}
	\Mat{X} = \Mat{A}^{-1/2}\big(\Mat{A}^{1/2} \Mat{B} \Mat{A}^{1/2}\big)^{1/2}\Mat{A}^{-1/2}\;.
	\label{eqn:riccati_sol}
\end{align}

According to Eq.~\ref{eqn:frechet_mean}, and by making use of the $J$ divergence, the Fr\'{e}chet mean must satisfy
\begin{equation}
	\frac{\partial \sum\nolimits_{i=1}^N \delta_J^2(\Mat{X}_i,\Mat{M})}{\partial \Mat{M}} = 0\;.
	\label{eqn:proof_mean_j1}
\end{equation}
Given that 
\begin{equation*}
	\frac{\partial \tr(\Mat{X}\Mat{M}^{-1})}{\partial \Mat{M}} = 
    \Mat{M}^{-1}\Mat{X}\Mat{M}^{-1}\;,
\end{equation*}
we have
\begin{align*}
	&\frac{\partial \sum\nolimits_{i=1}^N \delta_J^2(\Mat{X}_i,\Mat{M})}{\partial \Mat{M}} = 
	\sum\limits_{i=1}^N \Mat{X}_i^{-1} - \sum\limits_{i=1}^N\Mat{M}^{-1}\Mat{X}_i\Mat{M}^{-1} = 0\\
	&\Leftrightarrow \Mat{M} \sum\limits_{i=1}^N \Mat{X}_i^{-1} \Mat{M} = \sum\limits_{i=1}^N \Mat{X}_i\;,
\end{align*}
which is a \textit{Riccati} equation with a unique and closed form solution. A slightly different proof is also provided in~\cite{Vemuri_CVPR_2004}.
\end{proof}
\end{theorem}

\begin{remark}
There is a subtle difference between PCA in Euclidean space and the solution developed here. More specifically, unlike PCA
in Euclidean space $\Mat{W}^T \Mat{M} \Mat{W}$ does not necessarily represent the mean of the transformed data in $\SPD{m}$. 
That is,
\begin{equation*}
\Mat{W}^T \Mat{M} \Mat{W} \neq \underset{\Mat{M} \in \SPD{m}}{\arg\min}~~\sum_{i=1}^N \delta^2(\Mat{W}^T\Mat{X}_i\Mat{W},\Mat{M})
\end{equation*}
in general.
\end{remark}

%% file: sec_opt.tex
\subsection{Optimization Framework}
\label{sec:opt}

Both the unsupervised and supervised DR techniques introduced above can be viewed as solving an
optimization problem with a unitary constraint, which can generally be written as
\begin{align}
	\label{eqn:opt_main_equ}
	&\underset{\Mat{W}}{{\min}}~f(\Mat{W}) \notag \\
	&{\rm s.t.}~\Mat{W}^T\Mat{W} = \mathbf{I}_m\;,
\end{align}
where $f(\Mat{W})$ is the desired cost function and $\Mat{W} \in \mathbb{R}^{n \times m}$. In Euclidean space 
problems of the form of~\eqref{eqn:opt_main_equ} are typically cast as eigenvalue 
problems (\eg,~\cite{Saul_JMLR_2003,LI_TNN_2006,Yan_PAMI_2007,Jia_TNN_2009,Kokiopoulou_2011}). However, the complexity of 
our cost functions prohibits us from doing so. Instead, we propose to make use of manifold-based optimization techniques.

Recent advances in optimization methods formulate problems with unitary constraints as optimization problems on
Stiefel or Grassmann manifolds~\cite{Edelman_1998,Absil_2008}. More specifically, the geometrically correct setting for the  
minimization problem in~\eqref{eqn:opt_main_equ} is, in general, on a Stiefel manifold. 
However, if the cost function $f(\Mat{W})$ is independent from the choice of basis spanned by $\Mat{W}$, that is
if $f(\Mat{W}) = f(\Mat{W}\Mat{R})$ for $\Mat{R} \in \mathcal{O}(m)$,  then the problem is on a Grassmann manifold.
Here $\mathcal{O}(m)$ denotes the group of $m \times m$ orthogonal matrices. 
In our case, because of the affine invariance of the AIRM, the Stein divergence and the Jeffrey divergence, it can easily be checked that our cost function is indeed independent of the choice of basis. We can therefore make use of Grassmannian optimization techniques, and, in particular, of Newton-type optimization, which we briefly review below.

A Grassmann manifold $\GRASS{m}{n}$ is the space of $m$-dimensional linear subspaces of $\mathbb{R}^n$ for $0<m<n$~\cite{Absil_2008}.
Newton-type optimization, such as conjugate gradient (CG), over a Grassmannian is an iterative optimization routine that relies on a notion of search direction. In $\mathbb{R}^n$, such a direction is determined by the gradient vector. Similarly, on an abstract Riemannian manifold $\mathcal{M}$, 
the gradient of a smooth function identifies the direction of maximum ascent. 
Furthermore, the gradient of $f$ at a point
$x \in \mathcal{M}$, denoted by $\nabla f(x)$, is the element of $T_x\mathcal{M}$ satisfying
$\langle \nabla f(x), \zeta \rangle_x = Df(x)[\zeta],~\forall \zeta \in T_x\mathcal{M}$. 
Here, $\langle \cdot, \cdot \rangle_x$ is the Riemannian metric at $x$ and $Df(x)[\zeta]$ denotes
the directional derivative of $f$ at $x$ along direction $\zeta$.

On $\GRASS{m}{n}$, the gradient is expressed as
\begin{equation}
\nabla_\Mat{W}f(\Mat{W}) = (\mathbf{I}_n - \Mat{W}\Mat{W}^T)D_\Mat{W}(f),
\label{eqn:grad_euc_2_grass}
\end{equation}
where $\operatorname{grad}{f}(\Mat{W})$ is the $n \times m$ matrix of partial derivatives of ${f}(\Mat{W})$ 
with respect to the elements of $\Mat{W}$, \ie,
\begin{equation*}
[D_\Mat{W}(f)]_{i,j} = \frac{\partial f(\Mat{W})}{\partial\Mat{W}_{i,j}}.
\end{equation*}
For our approach, these derivatives are given by Eqs.~\ref{eqn:jacobian_tl_stein},~\ref{eqn:jacobian_tl_jeffrey} and~\ref{eqn:jacobian_tl_airm} for the Stein divergence, the Jeffrey divergence and the AIRM, respectively.

\def \GD_M_SCALE{0.750}
\begin{figure}[!t]
\centering
	\includegraphics[width= \GD_M_SCALE \columnwidth,keepaspectratio]{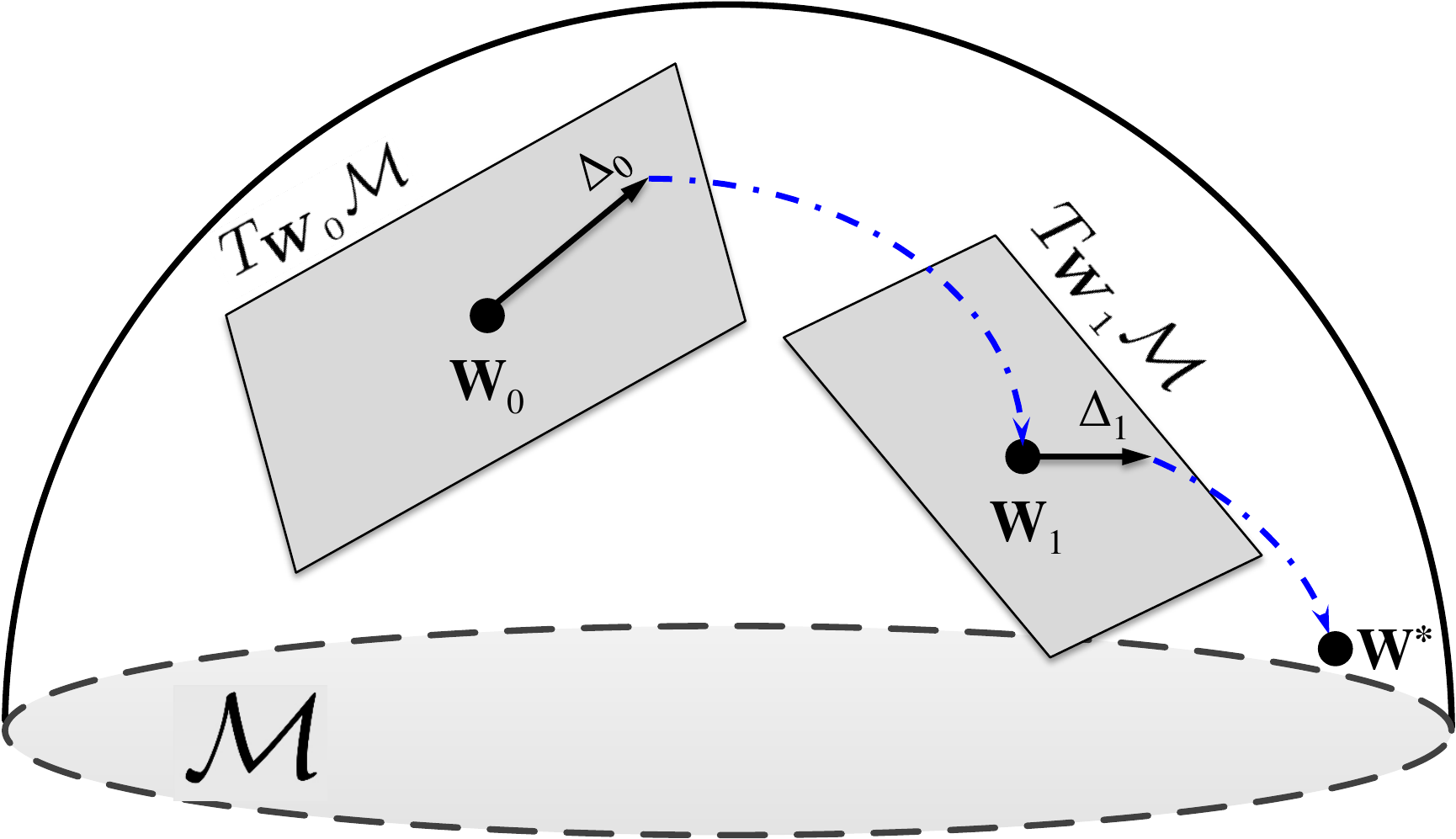}
	\caption{\small \textbf{Newton-type optimization on Riemannian manifolds.}  
	Here $\mathcal{M}$ denotes an abstract Riemannian manifold and $T_{\Mat{W}}\mathcal{M}$ is the tangent space of $\mathcal{M}$ at $\Mat{W}$.
	$\Delta$ represents the gradient of the cost function $f$, for example, $\Delta_0$ is the gradient of $f$ at $\Mat{W}_0$. In each 
	iteration of a Newton-type method, the gradient of the cost function is evaluated and a descent direction 
	is determined (for the steepest descent it is simply  along the gradient). The descent direction is mapped back to the manifold 
	via the exponential map (or a retraction) to identify the new solution. The aforementioned procedure continues until convergence.
	}
\label{fig:gd_M_drawing}
\end{figure}	

The descent direction obtained via $\nabla_\Mat{W}f(\Mat{W})$ needs to be mapped back to the manifold by 
the exponential map or by a retraction (see Chapter 4 in~\cite{Absil_2008} for definitions and detailed explanations).
In the case of the Grassmannian, this can be understood as forcing the unitary constraint while making sure that the cost function decreases. Fig.~\ref{fig:gd_M_drawing} provides a conceptual diagram for Newton-type optimization on Riemannian manifolds. 

Here, in particular, we make use of a CG method on the Grassmannian. CG methods compute the new descent direction by combining the gradient at the current and the previous solutions. To this end, it requires transporting the previous gradient to the current point on the manifold.
Unlike in flat spaces, on a manifold one cannot transport a tangent vector $\Delta$ from one 
point to another point by simple translation. To get a better intuition, take the case where the manifold is a sphere, and consider
two tangent spaces, one located at the pole and one at a point on the equator. Obviously the tangent vectors at the pole 
do not belong to the tangent space at the equator. Therefore, simple vector translation is not sufficient. 
As illustrated in Fig.~\ref{fig:parallel_transport}, transporting 
$\Delta$ from $\Mat{W}$ to $\Mat{V}$ on the manifold $\mathcal{M}$ requires subtracting the normal component 
$\Delta_{\perp}$ at $\Mat{V}$ for the resulting vector to be a tangent vector. Such a transfer of tangent vector 
is called {\it parallel transport}. On the Grassmann manifold, parallel transport, and the other operations required for a CG method, have efficient numerical forms, which makes them well-suited to perform optimization on the manifold.

CG on a Grassmann manifold can be summarized by the following steps:
\begin{itemize}
	\item[\bf (i)] Compute the gradient $\nabla_\Mat{W} f(\Mat{W})$ of the objective function $f(\Mat{W})$ on the manifold at 
	the current solution using
	\begin{equation}
		\nabla_\Mat{W} f(\Mat{W}) = (\mathbf{I}_n - \Mat{W} \Mat{W}^T) D_\Mat{W}(f)\;.
		\label{eqn:gradient_manif}
	\end{equation}
		
	\item[\bf (ii)] Determine the search direction $\Mat{H}$ by parallel transporting the previous search 
	direction and combining it with $\nabla_\Mat{W} f(\Mat{W})$.
		
	\item[\bf (iii)] Perform a line search along the geodesic at $\Mat{W}$ in the direction $\Mat{H}$.
	On the Grassmann manifold, the geodesics going from point $\Mat{X}$ in direction $\Delta$ can be represented
	by the Geodesic Equation~\cite{Absil_2008}
	\begin{equation}
		\Mat{X}(t) =
		\begin{bmatrix}
			\Mat{X}\Mat{V} &\Mat{U}
		\end{bmatrix}
		\begin{bmatrix}
			\cos (\Sigma t) \\
			\sin (\Sigma t)
		\end{bmatrix}
		\Mat{V}^T
		\label{eqn:geodesic_curve}
	\end{equation}
	where $t$ is the parameter indicating the location along the geodesic, and $\Mat{U} \Sigma \Mat{V}^T$
	is the compact singular value decomposition of $\Delta$. 
\end{itemize}
These steps are repeated until convergence to a local minimum, or until a maximum number of iterations is reached.

\begin{figure}[!tb]
	\centering
	\includegraphics[width = 0.6\columnwidth]{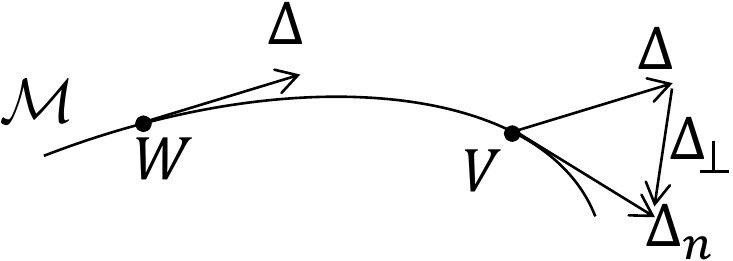}
	\caption{Parallel transport of a tangent vector $\Delta$ from a point $\Mat{W}$ to another point $\Mat{V}$ on the manifold.}
	\label{fig:parallel_transport}	
\end{figure}

It is worth mentioning that optimization techniques on matrix manifolds (\eg, Stiefel, Grassmann) 
are the core of several recent DR schemes~\cite{Cunningham_JMLR_2015,Huanglog_ICML_2015,Huang_2015_CVPR}. 
This is in part due to the availability of the \emph{manopt} package, which makes optimizing over various Riemannian manifolds 
simple and straight-forward~\cite{manopt_package}. As a matter of fact, in our experiments, we used the implementation of the Grassmannian CG method provided 
by \emph{manopt} to obtain $\Mat{W}$. Note that \emph{manopt} also provides other methods, such as trust-region solvers. A full evaluation of these solvers, however, goes beyond the scope of this paper.

%% file: sec_further_discussions.tex
\section{Further Discussions}
\label{sec:discussions}

In this section, we discuss several points regarding our DR framework. In particular, 
we discuss the case of the log-Euclidean metric~\cite{Arsigny_2006} and derive a formulation for this metric. Furthermore, we discuss the specific case where the SPD matrices encode Region Covariance Descriptors~\cite{Tuzel_ECCV_2006}.

\subsection{DR with the Log-Euclidean Metric}
\label{subsec:logeuc_dr}
In Section~\ref{sec:DR}, we have developed DR methods on $\SPD{n}$ based on the AIRM and on two Bregman divergences.
Another widely used metric to compare SPD matrices is the log-Euclidean metric defined as
\begin{equation}
	\delta_{lE}(\Mat{X} , \Mat{Y}) = \|\log(\Mat{X}) - \log(\Mat{Y})\|_F\;,
\end{equation}
where $\log(\cdot)$ denotes the matrix principal logarithm. This metric is indeed a true Riemannian metric (for a zero-curvature manifold)
and can be understood as a metric over the flat space that identifies the Lie algebra of an SPD manifold.
Below, we develop a supervised DR method on SPD manifolds similar to the one in Section~\ref{sec:unsupervised_tl}, but using log-Euclidean metric. The adaptation to
the unsupervised scenario introduced in Section~\ref{sec:unsupervised_tl} can easily be derived in a similar manner.

With the log-Euclidean metric,~\eqref{eqn:tl_main_equ} can be rewritten as
\begin{align}
	&\underset{\Mat{W} \in \mathbb{R}^{n \times m}}{\min}\sum_{i,j=1}^p \hspace{-1ex} 
	a(\Mat{X}_i,\Mat{X}_j) \Big \| \log(\Mat{W}^T\Mat{X}_i \Mat{W}) 
	\hspace{-0.5ex} - \hspace{-0.5ex} 	\log(\Mat{W}^T \Mat{X}_j \Mat{W}) \Big \|_F^2, \notag\\
	&\mathrm{s.t.}~ \Mat{W}^T\Mat{W} = \mathbf{I}_m\;. 
	\label{eqn:da_tl_log_euc0}
\end{align}
A difficulty in tackling~\eqref{eqn:da_tl_log_euc0} arises from the fact that an analytic form for the gradient of 
$\| \log(\Mat{W}^T\Mat{X}_i \Mat{W}) - 	\log(\Mat{W}^T \Mat{X}_j \Mat{W}) \Big \|_F^2$ with respect to $\Mat{W}$ 
is not known~\cite{Yger_Arxiv_2015}. To overcome this limitation, we introduce an approximation of $\log(\Mat{W}^T \Mat{X} \Mat{W})$. This approximation relies on the following lemma.

\begin{lemma}
The term $\log(\Mat{W}^T \Mat{X} \Mat{W})$ can be approximated as $\Mat{W}^T \log(\Mat{X}) \Mat{W}$.
\end{lemma}
\begin{proof}
Note that the Taylor expansion of $\log (\mathbf{I}_n - \Mat{A})$ is given by~\cite{Cheng_SIAM_2001},
\begin{equation}
\log (\mathbf{I}_n - \Mat{A}) = -\Mat{A} - \frac{1}{2}\Mat{A}^2 - \frac{1}{3}\Mat{A}^3 - \cdots.
\label{eq:taylor_log}
\end{equation}
Therefore, we can write
\begin{align*}
\log(\Mat{W}^T \Mat{X} \Mat{W}) &= \log(\mathbf{I}_n - (\mathbf{I}_n - \Mat{W}^T \Mat{X} \Mat{W}))\\
&\approx -(\mathbf{I}_n - \Mat{W}^T \Mat{X} \Mat{W})) = -\Mat{W}^T (\mathbf{I}_n - \Mat{X} )\Mat{W}\\
&\approx \Mat{W}^T \log(\Mat{X}) \Mat{W},
\end{align*}
where both the second and third lines make use of the first order Taylor approximation from Eq.~\ref{eq:taylor_log}.
\end{proof}

From the lemma above, we can cast~\eqref{eqn:da_tl_log_euc0} into the optimization problem
\begin{align}
	&\underset{\Mat{W} \in \mathbb{R}^{n \times m}}{\min} \sum_{i,j=1}^p \hspace{-1ex} a(\Mat{X}_i,\Mat{X}_j)
	\Big \| \Mat{W}^T \hspace{-0.5ex} \log(\Mat{X}_i)\Mat{W} \hspace{-0.5ex} - \hspace{-0.5ex}	\Mat{W}^T \hspace{-0.5ex} \log(\Mat{X}_j)\Mat{W} \Big \|_F^2, \notag\\
	&\mathrm{s.t.}~ \Mat{W}^T\Mat{W} = \mathbf{I}_m\;. 
	\label{eqn:da_tl_log_euc1}
\end{align}

The objective function of~\eqref{eqn:da_tl_log_euc1} can then be written as
\begin{align*}
	&\sum_{i,j=1}^p a(\Mat{X}_i,\Mat{X}_j) \Big \| \Mat{W}^T \log(\Mat{X}_i)\Mat{W} 
	- 	\Mat{W}^T \log(\Mat{X}_j)\Mat{W} \Big \|_F^2\\ 
	&= \tr{ \Big( \Mat{W}^T \Mat{F}(\Mat{W}) \Mat{W} \Big)} \;, 
\end{align*}
with
\begin{align}
	\Mat{F}(\Mat{W}) = \sum_{i,j=1}^p &a(\Mat{X}_i,\Mat{X}_j) \Big( \log(\Mat{X}_i) - \log(\Mat{X}_j) \Big) \Mat{W}\Mat{W}^T \times \notag\\
	&\Big(\log(\Mat{X}_i) - \log(\Mat{X}_j) \Big)\;, 
	\label{eqn:F_W_log}
\end{align}
which yields the optimization problem
\begin{align}
	&\underset{\Mat{W} \in \mathbb{R}^{n \times m}}{\min} \tr{ \Big( \Mat{W}^T \Mat{F}(\Mat{W}) \Mat{W} \Big)}, \notag\\
	&\mathrm{s.t.}~ \Mat{W}^T\Mat{W} = \mathbf{I}_m\;. 
	\label{eqn:da_tl_log_euc2}
\end{align}

We note that $\tr{ \big( \Mat{W}^T \Mat{F}(\Mat{W}) \Mat{W} \big)}$ is invariant to the action of the orthogonal group, \ie, 
changing $\Mat{W}$ with $\Mat{W}\Mat{R}$ for $\Mat{R} \in \mathrm{O}(m)$ does not change the value of the trace.
As such, in principle,~\eqref{eqn:da_tl_log_euc2} is a problem on $\GRASS{m}{n}$ and can be optimized in a similar manner as discussed before. In particular, to perform Newton-type methods on the Grassmannian, the required gradient is given by
\begin{align*}
&D_\Mat{W}\tr{ \Big( \Mat{W}^T \Mat{F}(\Mat{W}) \Mat{W} \Big)} = 
4 \sum_{i,j=1}^p a(\Mat{X}_i,\Mat{X}_j) \times\\
&\Big(\log(\Mat{X}_i) - \log(\Mat{X}_j) \Big) \Mat{W}\Mat{W}^T \Big(\log(\Mat{X}_i) - \log(\Mat{X}_j) \Big) \Mat{W}\;.
\end{align*}

While optimization on the Grassmannian can indeed be employed to solve~\eqref{eqn:da_tl_log_euc2}, here, we propose a faster alternative which relies 
on eigen-decomposition. To this end, we follow an iterative two-stage procedure. First, we fix $\Mat{F}(\Mat{W})$ (\ie, assume that it does not on $\Mat{W}$), and compute the solution of the resulting approximation of~\eqref{eqn:da_tl_log_euc2}, which can be achieved by computing the $m$ smallest eigenvectors of $\Mat{F}(\Mat{W})$~\cite{Kokiopoulou_2011}. Given the new $\Mat{W}$, we update $\Mat{F}(\Mat{W})$, and iterate. The pseudo-code of this procedure is given in Algorithm~\ref{alg:iterative_eigen_alg}. 

\begin{algorithm}[!tb]		
	\caption{\small Iterative Eigen-Decomposition Solver for log-Euclidean-based Supervised DR.}
	\label{alg:iterative_eigen_alg}
	\begin{algorithmic}
	\STATE {\bfseries Input:}\\
	{A set of SPD matrices $\{ \Mat{X}_i \}_{i=1}^p,~\Mat{X}_i~\in~\SPD{n}$\\ 
	The corresponding labels $\{ y_i \}_{i=1}^p,~y_i \in \{1,2,\cdots,C\}$\\
    The dimensionality $m$ of the induced manifold
	}
	\vspace{0.1cm}
	\STATE {\bfseries Output:}\\
	{The mapping $\Mat{W} \in \mathcal{G}(m,n)$
	}
	\vspace{0.25cm}
	\STATE Generate $a(\Mat{X}_i,\Mat{X}_j)$ using Eq.~\ref{eqn:affinity_Gw_Gb_combined}
	\STATE $\Mat{W} \gets \mathbf{I}_{n\times m}$ (\ie, the truncated identity matrix)
	\REPEAT		
	    \STATE Compute $F(\Mat{W})$ using Eq.~\ref{eqn:F_W_log}
		\STATE $\Mat{W} \gets  m$ smallest eigenvectors of $F(\Mat{W})$.
	\UNTIL{convergence}
	\end{algorithmic}
\end{algorithm} 

Figure~\ref{fig:eigen_cg_comparison} compares the speed and convergence behavior of our iterative eigen-decomposition-based solution 
against the CG method on the Grassmannian. This figure was computed using the MOCAP dataset 
(see Section~\ref{sec:exp_action_rec} for details). 
First, note that the eigen-decomposition solution converges much faster than CG. While CG yields a slightly lower error, our experiments 
show that the eigen-decomposition solver is typically 10 times faster than CG on the Grassmannian, which, we believe, justifies its usage.

\begin{remark}
The recent work of Huang \etal~\cite{Huanglog_ICML_2015} introduced the idea of learning a log-Euclidean metric, which is related to our log-Euclidean-based supervised DR approach.
This work formulates DR  as the problem of finding a positive {\emph semi}-definite matrix $\Mat{Q} \in \SYM{n}$ that maximizes the discriminative power of pairs of samples according to
\begin{align*}
	\delta_{i,j}(\Mat{Q}) = \tr \Big( \Mat{Q} \big(\log(\Mat{X}_i) - \log(\Mat{X}_j) \big) 
	\big(\log(\Mat{X}_i) - \log(\Mat{X}_j) \big)	\Big)	\;. 
\end{align*}
In particular, $\Mat{Q}$ was forced to have rank $m$, and thus identifies a low-dimensional latent space. 
Obtaining $\Mat{Q}$ was then formulated as a \emph{log-det} problem~\cite{Huanglog_ICML_2015}. Our formulation, here, yields a much simpler optimization problem, and will thus be faster.
\end{remark}

\begin{figure}[!tb]		
	\centering
	\includegraphics[width = 0.7 \columnwidth]{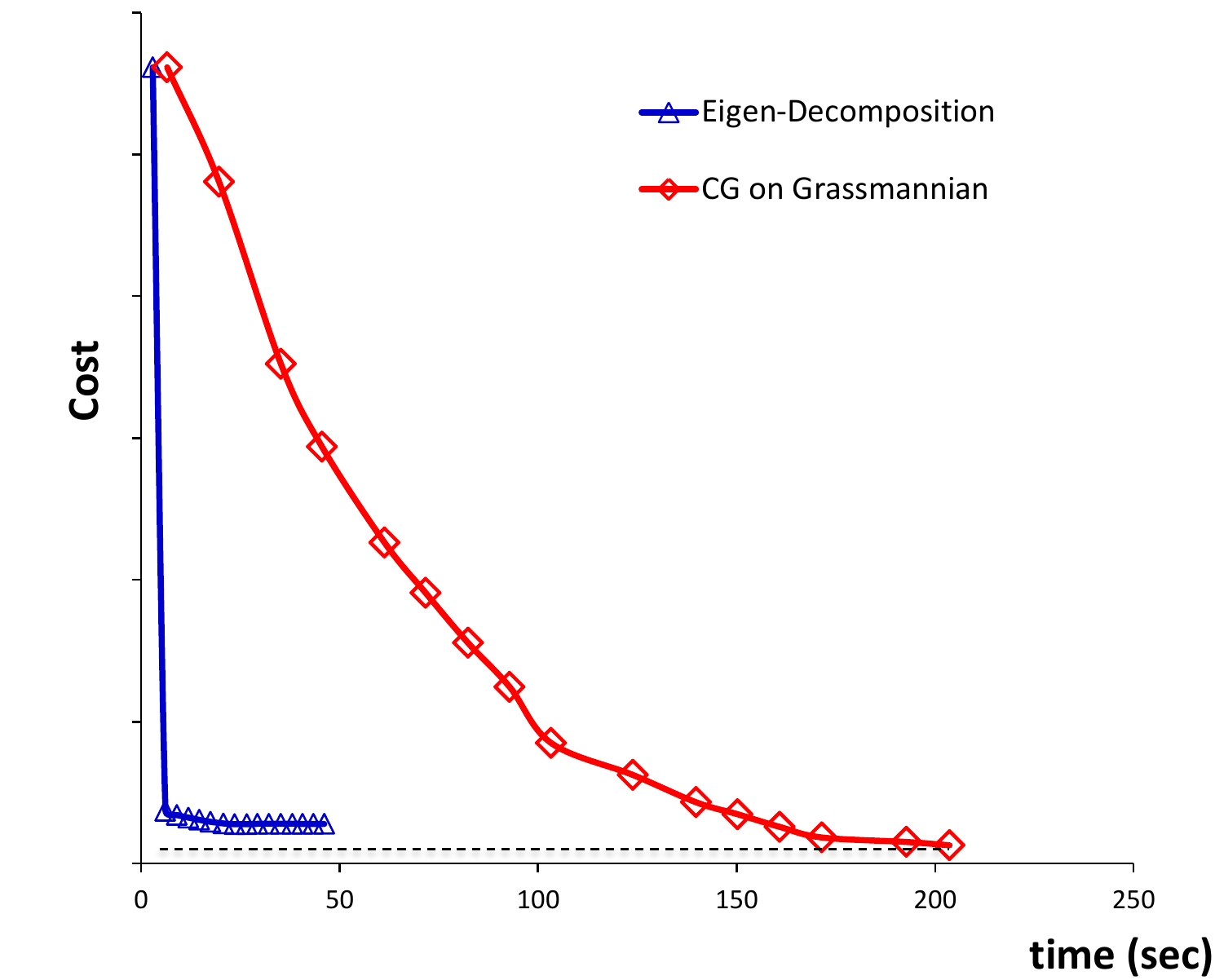}
	\caption{\small Convergence behavior of the proposed Eigen-Decomposition solver against conjugate gradient
	optimization on the Grassmann manifold.}
	\label{fig:eigen_cg_comparison}
\end{figure}

\begin{remark}
Since, in~\eqref{eqn:da_tl_log_euc1}, the $\log$ operation maps the matrices to Euclidean space, one might wonder if we should not apply a traditional vector-based DR approach to the resulting space. Note, however, that our goal is to go from $\SPD{n}$ to $\SPD{m}$. Therefore, optimizing a projection between the corresponding Euclidean spaces would translate to an optimization problem on $\GRASS{\frac{m\times(m+1)}{2}}{\frac{n\times(n+1)}{2}}$. By contrast, our symmetric formulation results in an optimization problem on $\GRASS{m}{n}$, which is clearly less expensive.
\end{remark}

\begin{remark}
From a purely geometrical point of view, we believe that the solutions developed using the AIRM, the Stein divergence and the Jeffrey divergence are more attractive. In particular, these solutions model the nonlinear geometry of the SPD manifold, while the log-Euclidean metric flattens it. Furthermore, in contrast with the log-Euclidean metric, the AIRM, the Stein and the Jeffrey divergences are invariant to affine transformations. We acknowledge, however, that the log-Euclidean metric has been shown to be a useful substitute to the AIRM in several applications (\eg,~\cite{Arsigny_2006,Wang_CVPR_2012_CDL,Carreira_TPAMI_15}).  
\end{remark}

\begin{remark}
Since the Frobenius norm also belongs to the family of Bregman divergences (with $\zeta(\Mat{X}) = \|\Mat{X}\|_F^2$), 
one could derive supervised/unsupervised DR formulations using $\|\cdot\|_F^2$ as similarity measure. For example, 
in the supervised case, this would translate to solving
\begin{align}
	&\underset{\Mat{W} \in \mathbb{R}^{n \times m}}{\min}~~ \sum_{i,j=1}^p \hspace{-1ex} a(\Mat{X}_i,\Mat{X}_j)
	\Big \| \Mat{W}^T \hspace{-0.5ex} \Mat{X}_i\Mat{W} \hspace{-0.5ex} - \hspace{-0.5ex}	\Mat{W}^T \hspace{-0.5ex} \Mat{X}_j\Mat{W} \Big \|_F^2, \notag\\
	&\mathrm{s.t.}~ \Mat{W}^T\Mat{W} = \mathbf{I}_m\;. 
	\label{eqn:da_tl_frob}
\end{align}
This can easily be rewritten in the form of~\eqref{eqn:da_tl_log_euc2}, but where now
$\Mat{F}(\Mat{W})$ has the 
\begin{align*}
	\Mat{F}(\Mat{W}) = \sum_{i,j=1}^p a(\Mat{X}_i,\Mat{X}_j) \Big(\Mat{X}_i - \Mat{X}_j \Big)\Mat{W}\Mat{W}^T 	
	\Big(\Mat{X}_i - \Mat{X}_j \Big)\;. 
\end{align*}
\end{remark}
Therefore, our previous eigen-decomposition solution directly applies here.

\subsection{Region Covariance Descriptors}
\label{sec:dr_covd}	

When it comes to the SPD matrices used in our experiments, we exploited Region Covariance Matrices (RCMs)~\cite{Tuzel_ECCV_2006} as image descriptors. Here, we discuss some interesting properties of our algorithm when applied to these specific SPD matrices.

There are several reasons why RCMs are attractive to represent images and videos. First, RCMs provide a natural way to fuse various feature types. Second, they help reducing the impact of noisy samples in a region via their inherent averaging operation. Third, RCMs are independent of the size of the region, and can therefore easily be utilized to compare regions of different sizes. Finally, RCMs can be efficiently computed using integral images~\cite{Tuzel_PAMI_2008,Sanin_WACV_2013}.

Let $I$ be a $W \times H$ image, and $\mathbb{O} = \{\Vec{o}_i\}_{i=1}^{r}, \; \Vec{o}_i \in \mathbb{R}^n$ be a set of $r$ observations extracted from $I$, \eg, $\Vec{o}_i$ concatenates intensity values, gradients along the horizontal and vertical directions, filter responses,... for image pixel $i$. Let $\mu = \frac{1}{r} \sum_{i = 1}^{r} \Vec{o}_i$ be the mean value of the observations.
Then, image $I$ can be represented by the $n \times n$ RCM 
\begin{align}
	\Mat{C}_{I} = \frac{1}{r-1} \sum_{i = 1}^{r} \left(\Vec{o}_i - \mu \right)\left(\Vec{o}_i - \mu \right)^T 
	= \Mat{O}\Mat{J}\Mat{J}^T\Mat{O}^T\;, 
	\label{eqn:cov_desc}
\end{align}
where $\Mat{J} = {r}^{-3/2}(r\mathbf{I}_{r} - \Mat{1}_{r \times r})$.
To have a valid RCM, $r \geq n$, otherwise $\Mat{C}_{I}$ would have zero eigenvalues, which would make both $\delta_g^2$ 
and $\delta_S^2$ indefinite. 

After learning the projection $\Mat{W}$, the low-dimensional representation of image $I$ is
given by $\Mat{W}^T\Mat{O}\Mat{J}\Mat{J}^T\Mat{O}^T\Mat{W}$.
This reveals two interesting properties of our learning scheme. {\bf 1)} The resulting representation can also be thought of as an RCM with $\Mat{W}^T\Mat{O}$ 
as a set of low-dimensional observations. Hence, in our framework, we can create a valid $\SPD{m}$ manifold with only $m$ observations instead of at least $n$ in the original input space. This is not the case for other algorithms, which require having matrices on $\SPD{n}$ as input.  {\bf 2)} Applying $\Mat{W}$ directly the set of observations reduces the computation time of creating the final RCM on $\SPD{m}$. This is due to the fact that the computational complexity of computing an RCM is quadratic in the dimensionality of the features. 

%% file: sec_related_work.tex
\section{Related Work}
\label{sec:related_work}

In this section, we review the methods that have exploited notions of Riemannian geometry for DR. In contrast with our approach that goes from one high-dimensional SPD manifold to a lower-dimensional manifold, most of the literature has focused on going from a manifold to Euclidean space.

In this context, a popular approach consists of flattening the manifold via its tangent space.
The best-known example of such an approach is Principal Geodesic Analysis (PGA)~\cite{Fletcher_CVPR_2003,Fletcher_2004_PGA}. PGA and its variants such as~\cite{Said_EUSIPCO_2007,Huckemann_2010,Sommer_ECCV_2010} have been successfully employed for various application, such as analyzing vertebrae outlines~\cite{Sommer_CVPRW_2009} and motion capture data~\cite{Said_EUSIPCO_2007}. PGA can be understood as a generalization of PCA to Riemannian manifolds. To this end, the widely-used 
formulation proposed in~\cite{Fletcher_2004_PGA} identifies the tangent space whose 
corresponding subspace maximizes the variability of the data on the manifold. PGA, however, is equivalent to flattening the Riemannian manifold by taking its tangent space at the Karcher, or Fr\'{e}chet, mean of the data. As such, it does not fully exploit the structure of the manifold. Furthermore, PGA, as PCA, cannot exploit the availability of class labels, and may therefore be sub-optimal for classification. 

Another recent popular trend consists of embedding the manifold to an RKHS to perform DR.
In particular,~\cite{Sadeep_CVPR_2013} relied on kernel PCA and~\cite{Wang_CVPR_2012_CDL} on
kernel Partial Least Squares (kPLS) and kernel Linear Discriminant Analysis (LDA) to achieve this goal.
Embedding the manifold to an RKHS inherently requires a p.d. kernel. While significant progress has been made 
in identifying p.d. kernels on Riemannian manifolds~\cite{Sadeep_CVPR_2013,Li_ICCV_2013,Feragen_2014}, our knowledge is still limited in this regard. 
For example and in the case of SPD manifolds,
the kernel employed in~\cite{Wang_CVPR_2012_CDL} is a linear kernel on the identity tangent space of $\SPD{n}$. 
In~\cite{Sadeep_CVPR_2013}, the best performing kernel corresponds to the Gaussian kernel defined on 
the identity tangent space of $\SPD{n}$. Therefore, in a very strict sense, the true structure of the manifold is not used by either of these works. As a matter of fact, it was recently shown that Gaussian kernels cannot preserve the geodesic distances on non-flat manifold~\cite{Feragen_2014}. 

In contrast to the previous methods, which flatten the manifold, via either a tangent space, or an RKHS,~\cite{Goh_CVPR_2008} directly employs notions of Riemannian geometry to perform nonlinear DR.
In particular,~\cite{Goh_CVPR_2008} extends several nonlinear DR techniques, such as Locally Linear Embedding (LLE), Hessian LLE and Laplacian Eigenmaps, to their Riemannian counterparts. 
Take for example the case of LLE~\cite{Saul_JMLR_2003}. Given a set of vectors $\{\Vec{x}_i\}_{i = 1}^m, \Vec{x}_i \in \mathbb{R}^D$, the LLE algorithm determines a weight matrix 
$\Mat{W} \in \mathbb{R}^{m \times m}$  which minimizes a notion of reconstruction error on  $\{\Vec{x}_i\}_{i = 1}^m$. 
Once the weight matrix $\Mat{W}$ is determined, the algorithm 
embeds $\{\Vec{x}_i\}_{i = 1}^m$ into  a lower dimensional space $\mathbb{R}^d, d < D$ where the neighbouring properties of 
$\{\Vec{x}_i\}_{i = 1}^m$ are preserved. The neighbouring properties are encoded by $\Mat{W}$ and the embedding takes the form of an 
eigen-decomposition in the end. As shown in~\cite{Goh_CVPR_2008}, the construction of $\Mat{W}$ 
can be generalized to the case of an arbitrary Riemannian manifold $\mathcal{M}$ by using the logarithm map.
Hence,  for a given set of points on $\mathcal{M}$, an embedding from $\mathcal{M} \to \mathbb{R}^d$ can be obtained once $\Mat{W}$ is appropriately constructed. In~\cite{Goh_CVPR_2008}, the authors showed on several clustering problems on $\mathcal{M}$ that the embedded data was more discriminative than the original one. In principle, the Riemannian extension of LLE (and of the other nonlinear DR algorithms discussed in~\cite{Goh_CVPR_2008}) can also be applied to classification problems. However, they are limited to the transductive setting since they do not define any parametric mapping to the low-dimensional space.

In contrast to the previous methods, whose learned mappings are to Euclidean space, a few recent techniques have studied the case of mapping between two manifolds of different dimensions. To this end, these techniques have also made use of the bilinear form of Eq.~\ref{eqn:generic_mapping}. In~\cite{Wang_2011_CCA}, a mapping between covariance matrices of different dimensions was learnt, but by ignoring the Riemannian geometry of the SPD manifold. More recently, and probably inspired by our preliminary study~\cite{Harandi_ECCV_2014}, this bilinear form was employed to perform DR on the SPD manifold and on the Grassmannian by exploiting notions of Riemannian geometry~\cite{Huanglog_ICML_2015,Huang_2015_CVPR,Yger_Arxiv_2015}. We also acknowledge that the work of Xu \etal~\cite{Xu_TCSVT_2015} is somehow relevant to the log-Euclidean development done in~\textsection\ref{subsec:logeuc_dr}. However, in contrast to our proposal, 
in~\cite{Xu_TCSVT_2015} authors did not impose an orthogonality constraint on $\Mat{W}$. 

Finally, concepts of Riemannian geometry have also been exploited in the context of DR in Euclidean space. For instance, 
Lin and Zha~\cite{RML_PAMI_2008} exploit the idea that the input (Euclidean) data lies
on a low-dimensional Riemannian manifold. Recently, Cunningham and Ghahramani~\cite{Cunningham_JMLR_2015} revisited linear DR techniques and analyzed them using the geometry of Stiefel manifolds.

%% file: sec_experiments.tex
\section{Empirical Evaluation}
\label{sec:experiments}

We now evaluate our different SPD-based DR methods on several problems. We first consider the supervised scenario and present results on image and video classification tasks. 
We then turn to evaluating our unsupervised techniques for clustering on SPD manifolds.
In all our experiments, the dimensionality of the low-dimensional SPD manifold was determined by cross-validation. 

\subsection{Image/Video Classification}
\label{subsec:classification_results}

The supervised SPD-DR algorithm introduced in Section~\ref{sec:supervised_dr} allows us to obtain a low-dimensional, more discriminative SPD manifold from a high-dimensional one. Many different classifiers can then be used to categorize the data on this new manifold. In our experiments, we make use of two such classifiers. First, we employ a simple nearest neighbour classifier based on the manifold metric (AIRM, $S$ or $J$ divergence). This simple classifier clearly evidences the benefits of mapping the original Riemannian structure to a lower-dimensional one. Second, we make use of the Riemannian sparse coding algorithm of~\cite{Harandi_TNNLS_2015}. This algorithm exploits the notion of sparse coding to represent a query SPD matrix using a codebook of SPD matrices. In all our experiments, we formed the codebook purely from the training data, \ie, no dictionary learning was employed. Note that this algorithm relies on a kernel derived from either the $S$ divergence
or the log-Euclidean metric. 
We refer to the different algorithms evaluated in our experiments as:

\begin{itemize}
	\renewcommand{\labelitemi}{\scriptsize$$}
	\item \textbf{NN-AIRM:} AIRM-based Nearest Neighbour classifier.
	\item \textbf{NN-S:} $S$ divergence-based Nearest Neighbour classifier.	
	\item \textbf{NN-J:} $J$ divergence-based Nearest Neighbour classifier.
	\item \textbf{NN-lE:} log-Euclidean-based Nearest Neighbour classifier.
	\item \textbf{NN-AIRM-DR:} AIRM-based Nearest Neighbour classifier on the low-dimensional SPD manifold obtained with our approach.	
	\item \textbf{NN-S-DR:} $S$ divergence-based Nearest Neighbour classifier on the low-dimensional SPD manifold obtained with our approach.
	\item \textbf{NN-J-DR:} $J$ divergence-based Nearest Neighbour classifier on the low-dimensional SPD manifold obtained with our approach.
	\item \textbf{NN-lE-DR:} log-Euclidean-based Nearest Neighbour classifier on the low-dimensional SPD manifold obtained with our approach.
	\item \textbf{kSC-S:} kernel sparse coding~\cite{Harandi_CVPR_2015} using the $S$ divergence on the high-dimensional SPD manifold.
	\item \textbf{kSC-lE:} kernel sparse coding using the log-Euclidean metric on the high-dimensional SPD manifold.
	\item \textbf{kSC-S-DR:} kernel sparse coding using the $S$ divergence on the low-dimensional SPD manifold obtained with our approach.
	\item \textbf{kSC-lE-DR:} kernel sparse coding using the log-Euclidean metric on the low-dimensional SPD manifold obtained with our approach.
\end{itemize}

In addition to these methods, we also provide the results of the PLS-based Covariance Discriminant Learning (CDL) technique 
of~\cite{Wang_CVPR_2012_CDL}, as well as of the state-of-the-art baselines of each specific dataset.

In practice, to define the affinity function (see Section~\ref{sec:supervised_dr}), we set $\nu_w$ to the minimum number of points in each class and, to balance the influence of $g_w(\cdot,\cdot)$ and $g_b(\cdot,\cdot)$, choose $\nu_b \leq \nu_w$, with the specific value found by cross-validation. 

\subsubsection{Material Categorization}
\label{sec:exp_material_cat}

For the task of material categorization, we used the UIUC dataset~\cite{UIUC_Dataset}. 
The UIUC material dataset contains 18 subcategories of materials taken in the wild from 
four general categories (see Fig.~\ref{fig:UIUC_Dataset}): \textit{bark}, \textit{fabric}, \textit{construction materials}, and 
\textit{outer coat of animals}.
Each subcategory has 12 images taken at various scales. 
Following standard practice, half of the images from each subcategory was randomly chosen as training data, and the rest was 
used for testing. We report the average accuracy over 10 different random partitions.

Small RCMs, such as those used for texture recognition in~\cite{Tuzel_ECCV_2006}, are hopeless here due to the complexity of the task.
Recently, SIFT features~\cite{SIFT_IJCV_2004} have been shown to be 
robust and discriminative for material classification~\cite{UIUC_Dataset}. Therefore, we constructed RCMs of size $155\times155$
using 128 dimensional SIFT features (from grayscale images) and 27 dimensional color descriptors. 
To this end, we resized all the images to $400 \times 400$ and computed dense SIFT descriptors on a regular grid with 4 pixels spacing. The color descriptors were obtained by simply stacking colors from $3\times3$ patches centered at the grid points. Each grid point therefore yields one 155-dimensional observation $\Vec{o}_i$ in Eq.~\ref{eqn:cov_desc}. The parameters for this experiments were set to $\nu_w = 6$ (minimum number of samples in a class), $m = 20$ and $\nu_b = 3$ obtained by 5-fold cross-validation.

Table~\ref{tab:table_UIUC_performance} compares the performance of the studied algorithms. The performance of 
the state-of-the-art method on this dataset was reported to be 43.5\%~\cite{UIUC_Dataset}. The results show that appropriate manifold-based methods (\ie, kSC-S and CDL) with the original $155 \times 155$ RCMs already outperform this state-of-the-art, while nearest neighbour 
(\eg, NN-AIRM, NN-S) on the same manifold yields worse performance. However, after applying our learning algorithm, NN not only outperforms 
state-of-the-art significantly, but also outperforms both CDL and kSC, except for the log-Euclidean solution. 
For example, kSC using the $S$ divergence is boosted by near than 14\% by dimensionality reduction. 
The maximum accuracy of $66.6\%$ is obtained by kernel sparse coding on the learned SPD manifold (kSC-S-DR).

\def \UIUC_SCALE {0.19}
\begin{figure}[!t]		
	\includegraphics[width = \UIUC_SCALE \columnwidth]{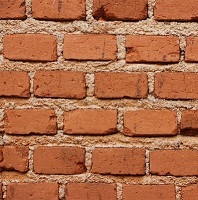}
	\includegraphics[width = \UIUC_SCALE \columnwidth]{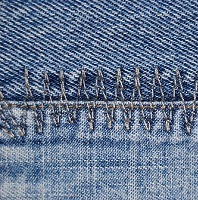}
	\includegraphics[width = \UIUC_SCALE \columnwidth]{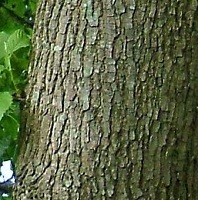}
	\includegraphics[width = \UIUC_SCALE \columnwidth]{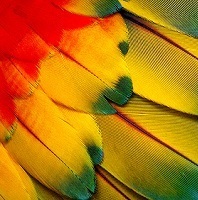}
	\includegraphics[width = \UIUC_SCALE \columnwidth]{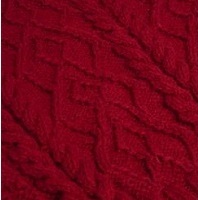}
	\caption{Samples from the UIUC material dataset~\cite{UIUC_Dataset}.}
	\label{fig:UIUC_Dataset}
\end{figure}

\begin{table*}
\scriptsize
\begin{minipage}[t]{0.3 \textwidth}
	\centering
   	\begin{tabular}{lc}
   		\toprule
    	{\bf Method} &{\bf Accuracy }\\
   		\toprule  		
    	{\bf CDL	}          		 	 &$52.3\% \pm 4.3$\\
   		\midrule
   		{\bf NN-AIRM}                         	 				 &$35.6\% \pm 2.6$\\
    	{\bf NN-AIRM-DR}                           				 &$58.3\% \pm 2.3$\\
   		\midrule
   		{\bf NN-S}                        	 					 &$35.8\% \pm 2.6$\\
    	{\bf NN-S-DR}	                          				 &$58.1\% \pm 2.8$\\
   		{\bf kSC-S	}        	 	 							 &$52.8\% \pm 2.1$\\  
   		{\bf kSC-S-DR}                         	 				 &$\bf{66.6\% \pm 3.1}$\\  		
   		\midrule
		{\bf NN-J}                        	 					 &$30.9\% \pm 2.4$\\
    	{\bf NN-J-DR}                          					 &$53.4\% \pm 2.9$\\
		\midrule
		{\bf NN-lE}                        	 					 &$36.7\% \pm 2.8$\\
    	{\bf NN-lE-DR}                          				 &$51.2\% \pm 3.0$\\
    	{\bf kSC-lE}        	 	 							 &$57.7\% \pm 4.2$\\
    	{\bf kSC-lE-DR}                     	 				 &$63.9\% \pm 4.3$\\		
   		\bottomrule	
	    \end{tabular}    	
   		\caption   {Recognition accuracies for the UIUC material dataset~\cite{UIUC_Dataset}.}	
   		\label{tab:table_UIUC_performance} 
\end{minipage}
\quad
\begin{minipage}[t]{0.3 \textwidth}
	\centering
	\begin{tabular}{lc}
		\toprule
    	{\bf Method} &{\bf Accuracy }\\
		\toprule  		
    	{\bf CDL}    					 &$79.8\%$\\		
    	\midrule
		{\bf NN-AIRM}                              				 &$62.8\%$\\
	   	{\bf NN-AIRM-DR}                           				 &$67.6\%$\\
	   	\midrule
    	{\bf NN-S}      	                         			 &$61.7\%$\\
		{\bf NN-S-DR}	                            			 &$68.6\%$\\
		{\bf kSC-S	}							        	 	 &$76.1\%$\\  
		{\bf kSC-S-DR}                         	 				 &$\bf{81.9\%}$\\
		\midrule
		{\bf NN-J}                        	 					 &$69.2\% $\\
    	{\bf NN-J-DR}                          					 &$71.8\% $\\
    	\midrule
    	{\bf NN-lE}                        	 					 &$69.7\% $\\
    	{\bf NN-lE-DR}                          				 &$71.3\% $\\
    	{\bf kSC-lE	}							        	 	 &$75.5\% $\\
    	{\bf kSC-lE-DR}                     					 &$78.7\% $\\                        	 			 
		\bottomrule	
	\end{tabular}
	\caption    { Recognition accuracies for the HDM05-MOCAP dataset~\cite{HDM05_Doc}.}	 	
	\label{tab:table_HDM05_performance}		   	
\end{minipage}	
\quad
\begin{minipage}[t]{0.3 \textwidth}
	\centering
	\begin{tabular}{lc}
		\toprule
    	{\bf Method} &{\bf Accuracy }\\
		\toprule  		
    	{\bf CDL}    					 &$70.9\%$\\		
    	\midrule
		{\bf NN-AIRM}                              				 &$64.7\%$\\
	   	{\bf NN-AIRM-DR}                           				 &$75.7\%$\\
	   	\midrule
    	{\bf NN-S}      	                         			 &$45.4\%$\\
		{\bf NN-S-DR}	                            			 &$72.8\%$\\
		{\bf kSC-S}        	 	 								 &$78.0\%$\\ 
		{\bf kSC-S-DR}                         	 				 &$\bf{80.1\%}$\\
		\midrule
		{\bf NN-J}                        	 					 &$62.0\% $\\
    	{\bf NN-J-DR}                          					 &$68.9\% $\\
    	\midrule                     	 			 
		{\bf NN-lE}                        	 					 &$39.8\% $\\
    	{\bf NN-lE-DR}                          				 &$55.0\% $\\
    	{\bf kSC-lE}        	 	 							 &$73.5\% $\\
    	{\bf kSC-lE-DR}                     	 				 &$78.8\% $\\	
		\bottomrule	
	\end{tabular}
	\caption    { Recognition accuracies for the YTC dataset~\cite{YT_Celebrity}.}	 	
	\label{tab:table_YTC_performance}		   	
\end{minipage}
\end{table*}

\subsubsection{Action Recognition from Motion Capture Data}
\label{sec:exp_action_rec}

As a second experiment, we tackled the problem of human action recognition from motion capture sequences using the HDM05 database~\cite{HDM05_Doc}. This database contains the following 14 actions: `clap above head', `deposit floor', `elbow to knee', `grab high', `hop both legs', `jog', `kick forward', `lie down floor', `rotate both arms backward', `sit down chair', `sneak', `squat', `stand up lie' and `throw basketball' (see Fig.~\ref{fig:MoCap_Dataset} for an example).
The dataset provides the 3D locations of 31 joints over time acquired at the speed of 120 frames per second.
We describe an action of a $K$ joints skeleton observed over $m$ frames by its joint covariance descriptor~\cite{Husse_IJCAI_2013}, which is an SPD matrix of size $3K \times 3K$. 
More specifically, let $x_i(t)$, $y_i(t)$ and $z_i(t)$ be the $x$, $y$, and $z$ coordinates of the $i^{th}$ joint at frame
$t$. Let $\Vec{s}(t)$ be the vector of all joint locations at time $t$, \ie, $\Vec{s}(t) = \big(x_1(t), \cdots, x_K(t),y_1(t),\cdots, y_K(t),z_1(t),  \cdots, z_K(t)\big)^T$, which has $3K$ elements.
The SPD matrix describing an action occurring over $\tau$ frames is then taken as the covariance of the vectors 
$\Vec{s}(t)\;,\; 1\leq t\leq \tau$.

In our experiments, we used 2 subjects for training (\ie, 'bd' and 'mm') and the remaining 3 subjects for testing (\ie, 'bk', 'dg' and 'tr')\footnote{Note that this differs from the setup in~\cite{Husse_IJCAI_2013}, where 3 subjects were used for training and 2 for testing. However, with the setup of~\cite{Husse_IJCAI_2013} where an accuracy of $95.41\%$ was reported, all our algorithms gave about $99\%$ accuracy, which made it impossible to compare them.}. This resulted in 118 training and 188 test sequences for this experiment. The parameters of our method were set to $\nu_w = 5$ (minimum number of samples in one class), $m = 65$ and $\nu_b = 5$ by cross-validation.

We report the performance of the different methods on this dataset in Table~\ref{tab:table_HDM05_performance}. 
Again we can see that the accuracies of NN and kSC are significantly improved by our learning algorithm, and that the kSC-S-DR 
approach achieves the best accuracy of $81.9\%$.

\def \MOCAP_SCALE  {0.7}
\begin{figure}[!tb]
	\centering
	\includegraphics[width = \MOCAP_SCALE \columnwidth]{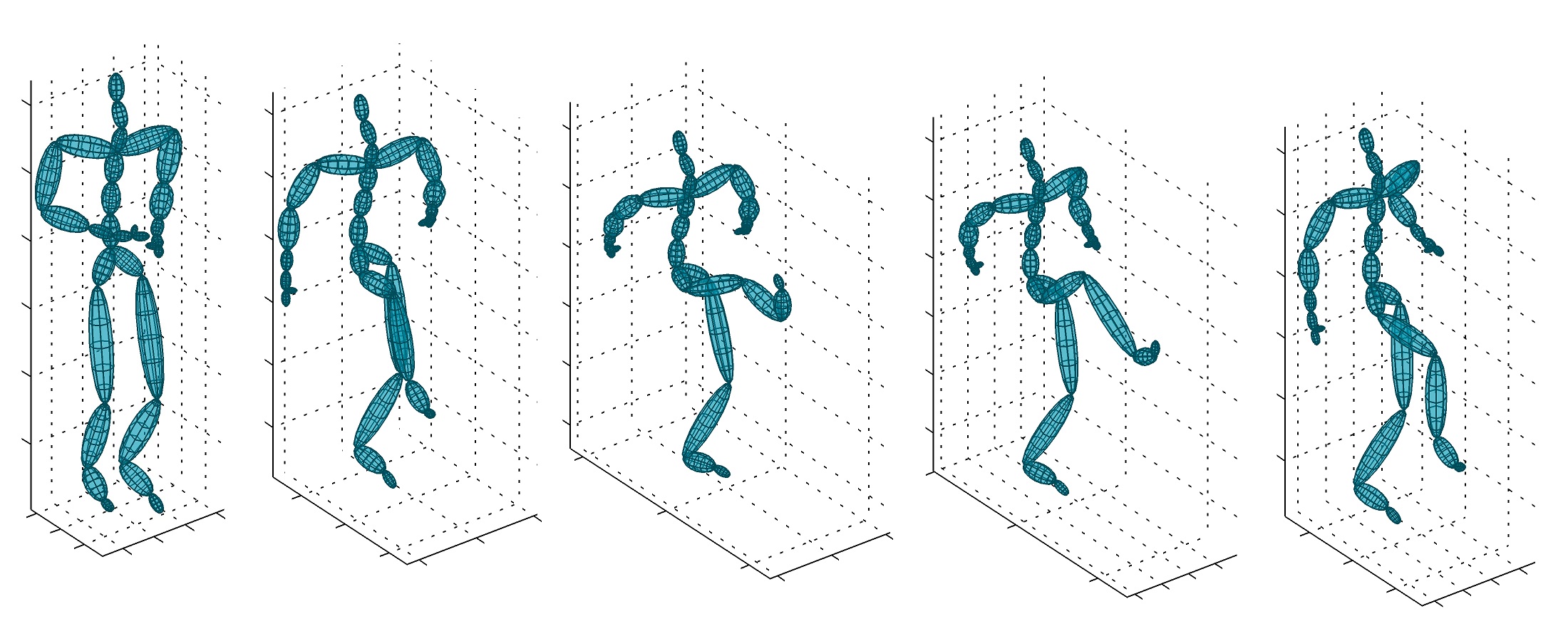}
	\caption{Kicking action from the HDM05 motion capture sequences database~\cite{HDM05_Doc}.}
	\label{fig:MoCap_Dataset}
\end{figure}

\subsubsection{Face Recognition}
\label{sec:exp_face_rec}

We then used the YTC dataset~\cite{YT_Celebrity} for the task of image-set-based face recognition. 
The YTC dataset contains 1910 video clips of 47 subjects. See Fig.~\ref{fig:YT_Celebrity_example} for samples from YTC. 
We used face regions extracted from the videos and resized them to $64 \times 64$. From each frame in a video, we then extracted 
4 histograms of Local Binary Patterns (LBP)~\cite{LBP_PAMI_2002}, each obtained from a $32 \times 32$ sub-region of the frame.
By concatenating the LBP histograms, frame $i$ in a video is described by $\Vec{o}_i$, a $232$-dimensional vector.
A video is then described by one SPD matrix of the form 
\begin{equation}
\Mat{C} = 
\begin{bmatrix}
\Mat{O}\Mat{O}^T + \mu \mu^T &\mu\\
\mu^T &1
\end{bmatrix}\;,
\label{eqn:rcm_ytc}
\end{equation}
where $\mathbb{R}^{232 \times r} \ni \Mat{O} = [\Vec{o}_1,\Vec{o}_2,\cdots,\Vec{o}_r]$ 
is a matrix storing the descriptors of all $m$ frames of a video and $\mu = \frac{1}{m} \sum_{i=1}^m\Vec{o}_i$.
Following the standard practice~\cite{Lu_ICCV_2013}, 3 videos from each person
were randomly chosen as training/gallery data, and the query set contained 6 randomly chosen videos from each subject. 
The process of random selection was repeated 5 times.

In Table~\ref{tab:table_YTC_performance}, we compare the performance of all the studied algorithm. To the best of our knowledge,
the highest reported accuracy using holistic descriptors (\ie, one descriptor per video) is $78.2\%$~\cite{Lu_ICCV_2013}. 
Both kSC methods after dimensionality reduction outperform this result, with kSC-S-DR achieving the maximum performance of $80.1\%$.
Note also that our DR scheme significantly boosts the performance of NN using the log-Euclidean and the Stein metrics (\eg, from $45.4\%$ to $72.8\%$ in the case of the Stein divergence).

\def \YTSIZE {0.225}
\begin{figure}[!tb] 
  \centering 
  \includegraphics[width=\YTSIZE \columnwidth,keepaspectratio]{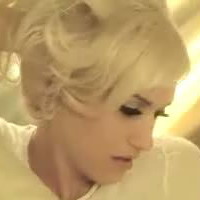}
  \hfill
  \includegraphics[width=\YTSIZE \columnwidth,keepaspectratio]{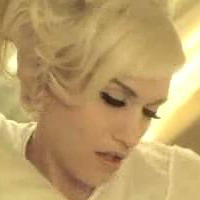}
  \hfill
  \includegraphics[width=\YTSIZE \columnwidth,keepaspectratio]{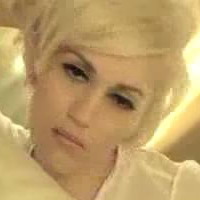}
  \hfill
  \includegraphics[width=\YTSIZE \columnwidth,keepaspectratio]{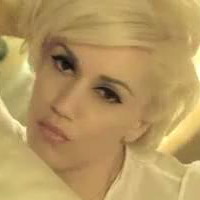}
  \caption
    {
    Samples from YouTube celebrity~\cite{YT_Celebrity}.
    }
  \label{fig:YT_Celebrity_example}
\end{figure}

\subsection{Video Clustering}
\label{subsec:clustering_results}

The unsupervised algorithm introduced in Section~\ref{sec:unsupervised_tl} allows us to obtain a low-dimensional SPD manifold 
from a high-dimensional one by maximizing a notion of data variance. We now evaluate the performance of 
this unsupervised DR approach on the task of video clustering. To this end, we report both the clustering accuracy and the Normalized Mutual Information (NMI)~\cite{Strehl:AAAI:2000}, which measures the amount of statistical information shared by random variables representing the cluster distribution and the underlying class distribution of the data points. Let $P_C$ be the random variable denoting the cluster assignments of the points
and $P_K$ the random variable denoting the underlying class labels on the points. Then, the NMI is defined as
\begin{equation}
NMI = 2\frac{I(P_C;P_K)}{H(P_C) + H(P_K)}\;,
\label{eqn:NMI}
\end{equation}
where $I(P_X;P_Y) = H(P_X) - H(P_X|P_Y)$ is the mutual information between the random variables $P_X$ and $P_Y$,
$H(P_X)$ is the Shannon entropy of $P_X$, and $H(P_X|P_Y)$ is the conditional entropy of $P_X$ given $P_Y$. The normalization by the average entropy of $P_C$ and $P_K$ makes the NMI be between 0 and 1.
For measuring the clustering accuracy, we followed the metric described in~\cite{Cai_TKDE_2005}. 
More specifically, for a query sample $\Mat{X}_i$, let $r_i$ and $s_i$ be the obtained cluster label and the ground truth label, respectively. The accuracy (AC) is defined as follows:
\begin{equation*}
\mathrm{AC} = \frac{1}{n}\sum_{i = 1}^{n} g(s_i,\mathrm{map}(r_i))\;.
\end{equation*}
where $n$ is the total number of queries, $g(x,y)$ is equal to one if
$x = y$ and zero otherwise, and $\mathrm{map}(r_i)$ is the permutation mapping function that
maps each cluster label $r_i$ to the equivalent label from the ground truth. The best mapping
can be found by using the Kuhn-Munkres algorithm~\cite{Lovasz_2009}.

For the task of clustering, we used the static setting of the UMD Keck body-gesture data set~\cite{KECK_Dataset},
which consists of 126 videos of 14 naval body gestures. Samples are shown in Fig.~\ref{fig:Keck_samples}.
We described each video in a similar manner as in the YTC experiment, albeit with a couple of differences.
More specifically, we used Histograms of Gradients (HoG)~\cite{Dalal_CVPR_2005} instead of LBP histograms to describe each frame.
Furthermore, each frame was resized to $32 \times 32$, and we concatenated HoG features extracted from $16 \times 16$ non-overlapped regions 
to form the frame descriptor. Using the idea of Eq.~\eqref{eqn:rcm_ytc} to aggregate the frame descriptors, we obtained an SPD matrices of size $125 \times 125$ to describe each video. 

For our evaluation, we employed the k-means algorithm on the manifold using the AIRM and the Jeffrey and the Stein divergences. 
We also made use of the k-means algorithm on the identity tangent space for the log-Euclidean metric. 
In addition to k-means on the manifold, we also utilized the kernel k-means algorithm using the Jeffrey, Stein and log-Euclidean kernels.
We refer to the algorithms evaluated in our experiments as:

\begin{itemize}
	\renewcommand{\labelitemi}{\scriptsize$$}
	\item \textbf{KM-AIRM:} k-means based on the AIRM on the high-dimensional SPD manifold.
	\item \textbf{KM-S:} k-means based on the $S$ divergence on the high-dimensional SPD manifold.	
	\item \textbf{KM-J:} k-means based on the $J$ divergence on the high-dimensional SPD manifold.
	\item \textbf{KM-lE:} k-means based on the log-Euclidean metric on the high-dimensional SPD manifold.
	\item \textbf{KM-AIRM-DR:} k-means based on the AIRM on the low-dimensional SPD manifold obtained with our approach.	
	\item \textbf{KM-S-DR:} k-means based on the $S$ divergence on the low-dimensional SPD manifold obtained with our approach.
	\item \textbf{KM-J-DR:} k-means based on the $J$ divergence on the low-dimensional SPD manifold obtained with our approach.
	\item \textbf{KM-lE-DR:} k-means based on the log-Euclidean metric on the low-dimensional SPD manifold obtained with our approach.
	\item \textbf{kKM-S:} kernel k-means based on the $S$ divergence on the high-dimensional SPD manifold.
	\item \textbf{kKM-lE:} kernel k-means based on the log-Euclidean metric on the high-dimensional SPD manifold.
	\item \textbf{kKM-S-DR:} kernel k-means based on the $S$ divergence on the low-dimensional SPD manifold obtained with our approach.
	\item \textbf{kKM-lE-DR:} kernel k-means based on the log-Euclidean metric on the low-dimensional SPD manifold obtained with our approach.
\end{itemize}

\def \KECK_SIZE {0.30}
\begin{figure}[!tb]
	\centering
  	\includegraphics[width = \KECK_SIZE \columnwidth,keepaspectratio]{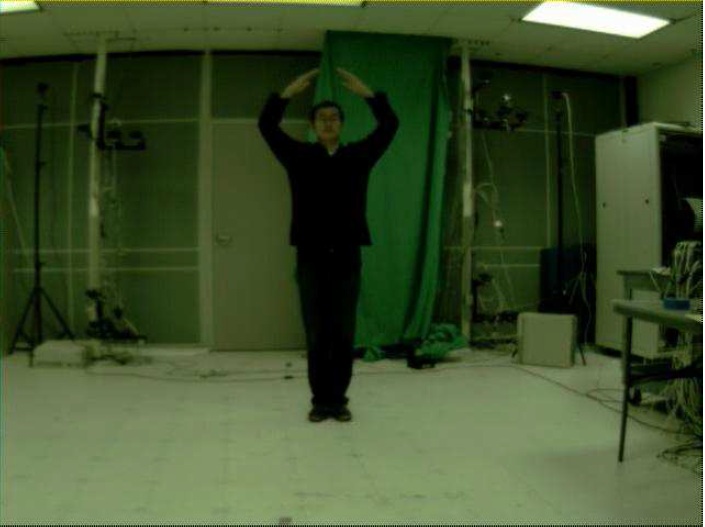}
   	\includegraphics[width = \KECK_SIZE \columnwidth,keepaspectratio]{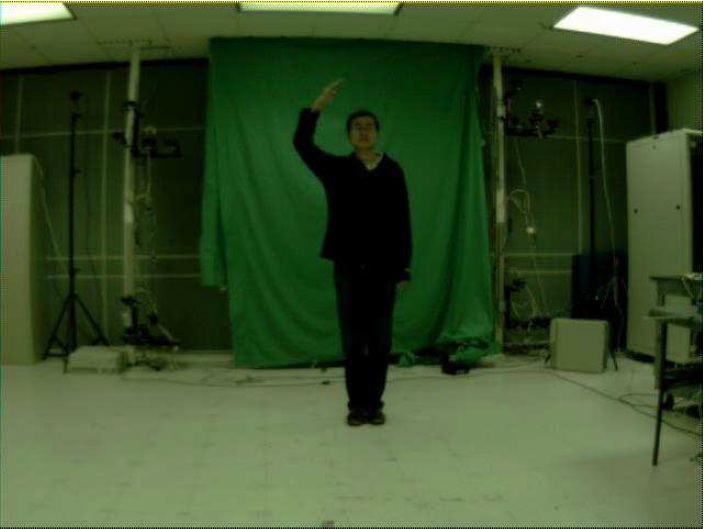}
   	\includegraphics[width = \KECK_SIZE \columnwidth,keepaspectratio]{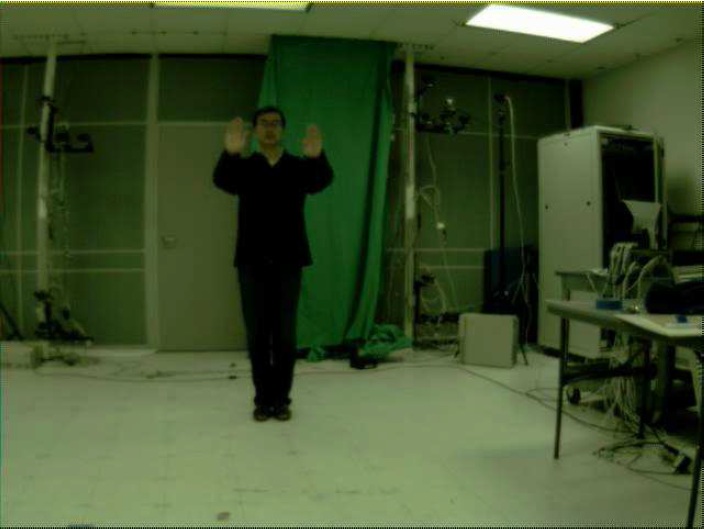}
    \caption{{\small Sample images from the UMD Keck body-gesture dataset~\cite{KECK_Dataset}.}}
	\label{fig:Keck_samples}	
\end{figure}

Table~\ref{tab:table_KECK_performance} reports the accuracy and NMI values for all the studied methods. 
It is interesting to see that the log-Euclidean metric achieves better accuracy on this dataset. We also note that the AIRM  
outperforms the solutions based on the Bregman divergences. However, with kernel k-means,  the Stein-based algorithm
surpasses the AIRM-based one.

\begin{table}
	\centering
   	\begin{tabular}{lcc}
   		\toprule
    	{\bf Method} &{\bf AC }	&{\bf NMI }\\
   		\toprule  		
   		{\bf AIRM}                         	 				 &$56.2\% \pm 1.2$			&$73.0\% \pm 1.8$\\
    	{\bf AIRM-DR}                           			 &$64.7\% \pm 2.0$			&$79.0\% \pm 1.7$\\
   		\midrule
   		{\bf KM-S}                        	 				 &$53.9\% \pm 1.4$			&$71.0\% \pm 1.1$\\
    	{\bf KM-S-DR}	                          			 &$61.4\% \pm 1.2$			&$78.7\% \pm 0.9$\\
   		{\bf kKM-S}		                        	 		 &$64.1\% \pm 1.8$			&$77.5\% \pm 0.1$\\
    	{\bf kKM-S-DR}	            		              	 &$71.2\% \pm 1.4$			&$83.7\% \pm 0.4$\\   		
    	\midrule
		{\bf KM-J}              	          	 			 &$53.8\% \pm 1.7$			&$71.2\% \pm 1.0$\\
    	{\bf KM-J-DR}	    	                      		 &$55.3\% \pm 2.1$			&$72.8\% \pm 0.9$\\
    	\midrule
    	{\bf KM-lE}                        	 				 &$62.7\% \pm 0.9$			&$79.2\% \pm 0.3$\\
    	{\bf KM-lE-DR}	                       			 	 &$75.3\% \pm 1.5$			&$88.3\% \pm 0.2$\\
   		{\bf kKM-lE}                        	 		 	 &$71.3\% \pm 1.7$			&$83.5\% \pm 0.2$\\
    	{\bf kKM-lE-DR}		                          	 	 &$83.2\% \pm 0.2$			&$91.8\% \pm 0.2$\\   		
   		\bottomrule	
	    \end{tabular}    	
   		\caption   {Recognition accuracies and normalized mutual information scores (mean and standard deviations) 
   		for the Keck dataset~\cite{KECK_Dataset}.}	
   		\label{tab:table_KECK_performance} 
\end{table}

%% file: sec_conclusion.tex
\section{Conclusions and Future Work}
\label{sec:conclusion}

We have introduced an approach to mapping a high-dimensional SPD manifold into a lower-dimensional one. In particular, we have derived both a supervised and an unsupervised formulation. In both cases, we have studied different metrics to encode the similarity between SPD matrices, namely, the AIRM, the Stein divergence, the Jeffrey divergence and the log-Euclidean metric. Our experiments have shown that reducing the dimensionality consistently improved accuracy over directly using the high-dimensional SPD matrices. In particular, we have found that the Stein divergence was particularly powerful in the supervised case, while the log-Euclidean metric was highly competitive in the unsupervised one. We believe that this work, extended our preliminary study~\cite{Harandi_ECCV_2014}, which already generated follow-ups~\cite{Huanglog_ICML_2015,Huang_2015_CVPR,Yger_Arxiv_2015}, is an important step towards developing DR algorithms dedicated to Riemannian manifolds, and in particular in the context of going from a high-dimensional manifold to a lower-dimensional one. In the future, we therefore intend to extend this framework to other types of Riemannian manifolds.

%% file: sec_appendix.tex
\section{The equivalency between the length of curves under AIRM and the J-divergence}
\label{app:details}

Here, we prove 
the equivalency between the length of any given curve under $\delta_R^2$ and $\delta_j^2$ up to scale of $\sqrt{2}$.
The proof of this theorem is developed in several steps.
We start with the definition of curve length and intrinsic metric.
Without any assumption on differentiability, let $(M,d)$ be a metric space.  
A curve in $M$ is a continuous function $\gamma : [0, 1] \rightarrow M$
and joins the starting point $\gamma(0) = x$ to the end point $\gamma(1) = y$.
Let us define the following:

\begin{definition} \label{def:curve_length}
	The length of a curve $\gamma$ is the supremum of $\ell(\gamma ; \{t_i \})$ over all possible partitions \mbox{$\{t_i \}$}
	with \mbox{{$\{t_i \}$}} satisfying \mbox{$0 = t_0 < t_1 < \cdots < t_{n-1} < t_n = 1$} and
	\mbox{{$\ell(\gamma ; \{t_i \}) = \sum_{i}d\left(\gamma(t_i),\gamma(t_{i-1})\right)$}}.
\end{definition}

\begin{definition} \label{def:intrinsic_metric_thm}
	The intrinsic metric $\widehat{\delta}$
	is defined as the infimum of the length of all paths from $x$ to $y$.
\end{definition}
	
\begin{theorem}
	If the intrinsic metrics induced by two metrics $d_1$ and $d_2$ are identical to scale $\xi$,
	then the length of any given curve is the same under both metrics up to $\xi$~\cite{Hartley_IJCV_13}.
\end{theorem}
\begin{theorem}
	If $d_1(x,y)$ and $d_2(x,y)$ are two metrics defined on a space $M$ such that
	\begin{equation}
		\lim_{d_1(x,y) \rightarrow 0} \:\frac{d_2(x,y)}{d_1(x,y)} = 1
		\label{eqn:intrinsic_metric0}
	\end{equation}
	uniformly (with respect to $x$ and $y$), then their intrinsic metrics are identical~\cite{Hartley_IJCV_13}.
\end{theorem}
	
Therefore, we need to study the behavior of 
\begin{align*}
	&\lim_{\delta_J^2(\Mat{X},\Mat{Y}) \rightarrow 0} \:\frac{\delta_R^2(\Mat{X},\Mat{Y})}{\delta_J^2(\Mat{X},\Mat{Y})}\;,
\end{align*}
to prove our theorem on curve length.

\begin{proof}
	We first note that for an affine invariant metric $\delta$ on $\SPD{n}$,
	$\delta(\Mat{X},\Mat{Y}) = \delta(\mathbf{I}_n,\Mat{D}^{-1/2}\Mat{L}^T\Mat{Y}\Mat{L}\Mat{D}^{-1/2})$, 
	where $\Mat{X} = \Mat{L}\Mat{D}\Mat{L}^T$ and $\Mat{L}\Mat{L}^T = \mathbf{I}_n$.
	As a result, we just 
	need to study the behavior of our metrics around $\mathbf{I}_n$ to draw any conclusion. The behavior of a point close to 
	$\mathbf{I}_n$ for an affine invariant metric can be 
	described by a diagonal matrix in the form $\DIAG(\exp (\nu_i))$. This can be understood by considering	
	the exponential map of a tangent vector $\Mat{U}\DIAG(\nu_i)\Mat{U}^T$
	at the identity tangent space and noting that
	$\delta(\mathbf{I}_n,\Mat{U}\DIAG(\exp (\nu_i))\Mat{U}^T) = \delta(\mathbf{I}_n,\DIAG(\exp (\nu_i))),
	~~\forall~\Mat{U}:\Mat{U}\Mat{U}^T = \Mat{U}^T\Mat{U} =\mathbf{I}_n$. For the $J$-divergence, we have
  
  \noindent
	\begin{align}
		\lim_{\Mat{X} \rightarrow \Mat{Y}} \:\frac{\delta_R^2 (\Mat{X},\Mat{Y})}{\delta_J^2(\Mat{X},\Mat{Y})} &= 
		\lim_{t \rightarrow 0} \:\frac{\delta_R^2 \Big(\mathbf{I}_n,\DIAG\big(\exp (t\nu_i)\big)\Big)}
		{\delta_J^2\Big(\mathbf{I}_n,\DIAG\big(\exp (t\nu_i)\big)\Big)}
		\notag \\
		&\hspace{-15ex}= \lim_{t \rightarrow 0} \:\frac
		{2\Big\| \log \Big( \DIAG\big( \exp(t \nu_i)\big) \Big)\Big\|_F^2}
		{\tr \Big \{\DIAG\big( \exp(t \nu_i)\big)
		+ \DIAG\big( \exp(-t \nu_i)\big) \Big\} -2n  }
		\notag \\
		&\hspace{-15ex}= \lim_{t \rightarrow 0} \:\frac
		{2t^2\sum_{i = 1}^n\nolimits \nu_i^2}
		{\sum_{i = 1}^n\nolimits \exp(t \nu_i)+\sum_{i = 1}^n\nolimits\exp(-t \nu_i) -2n}
		\label{eqn:proof_jeff_airm_1} \\
		&\hspace{-15ex}= \lim_{t \rightarrow 0} \:\frac
		{4\sum_{i = 1}^n\nolimits{\nu_i^2}}
		{\sum_{i = 1}^n\nolimits \nu_i^2\exp(t \nu_i) + \sum_{i = 1}^n\limits \nu_i^2\exp(-t \nu_i)}
		= 2
		\label{eqn:proof_jeff_airm_2}
	\end{align}
	where L'H\^{o}pital's rule was used twice from \eqref{eqn:proof_jeff_airm_1} to \eqref{eqn:proof_jeff_airm_2} since 
	the limit in \eqref{eqn:proof_jeff_airm_1} was indefinite. 	
	Therefore,  $\lim_{\Mat{X} \rightarrow \Mat{Y}} \:\frac{\delta_R (\Mat{X},\Mat{Y})}{\delta_J (\Mat{X},\Mat{Y})} = \sqrt{2}$,
	which concludes the proof.	
\end{proof}